\documentclass[10pt,journal,compsoc,twoside]{IEEEtran}
\usepackage{setspace}

\usepackage{amssymb,amsmath,amsthm}
\DeclareMathOperator*{\argmax}{arg\,max}
\DeclareMathOperator*{\argmin}{arg\,min}

\usepackage[linesnumbered,ruled]{algorithm2e}
\usepackage[noend]{algpseudocode}

\usepackage{array}

\usepackage[utf8]{inputenc}
\usepackage[english]{babel}
\usepackage{graphicx}  
\usepackage{subcaption}  
\theoremstyle{definition}

\newtheorem{theorem}{Theorem}

\newtheorem{lemma}{Lemma}
\theoremstyle{remark}

\usepackage[nocompress]{cite}
\usepackage{hyperref}
\usepackage[flushleft]{threeparttable}

\usepackage{hyphenat}
\begin{document}

\title{Efficient Global Multi-object Tracking Under Minimum-cost Circulation Framework}

\author{Congchao~Wang,
        Yizhi~Wang,
        and~Guoqiang~Yu
\IEEEcompsocitemizethanks{
\IEEEcompsocthanksitem C. Wang, Y. Wang and G. Yu are with the Bradley Department
of Electrical and Computer Engineering, Virginia Tech, Arlington,
VA, 22203.\protect\\
E-mail: \{ccwang, yzwang, yug\}@vt.edu
}
}


\IEEEtitleabstractindextext{%
\begin{abstract}
We developed a minimum-cost circulation framework for solving the global data association problem, which plays a key role in the tracking-by-detection paradigm of multi-object tracking (MOT). The global data association problem was extensively studied under the minimum-cost flow framework, which is theoretically attractive as being flexible and globally solvable. However, the high computational burden has been a long-standing obstacle to its wide adoption in practice. While enjoying the same theoretical advantages and maintaining the same optimal solution as the minimum-cost flow framework, our new framework has a better theoretical complexity bound and leads to orders of practical efficiency improvement. This new framework is motivated by the observation that minimum-cost flow only partially models the data association problem and it must be accompanied by an additional and time-consuming searching scheme to determine the optimal object number. By employing a minimum-cost circulation framework, we eliminate the searching step and naturally integrate the number of objects into the optimization problem. By exploring the special property of the associated graph, that is, an overwhelming majority of the vertices are with unit capacity, we designed an implementation of the framework and proved it has the best theoretical computational complexity so far for the global data association problem. We evaluated our method with 40 experiments on five MOT benchmark datasets. Our method was always the most efficient in every single experiment and averagely 53 to 1,192 times faster than the three state-of-the-art methods. When our method served as a sub-module for global data association methods utilizing higher-order constraints, similar running time improvement was attained. We further illustrated through several case studies how the improved computational efficiency enables more sophisticated tracking models and yields better tracking accuracy. We made the source code publicly available on GitHub with both Python and MATLAB interfaces.

\end{abstract}

}

\maketitle

\IEEEdisplaynontitleabstractindextext

%
\IEEEpeerreviewmaketitle

\IEEEraisesectionheading{\section{Introduction}\label{sec:introduction}}
\IEEEPARstart{M}{ulti-object} tracking (MOT) is a fundamental task in machine intelligence with a variety of applications such as traffic surveillance, autonomous driving, particle tracking, and cell lineage analysis \cite{motsurvey1,motsurvey2}. 
The continuing improvement of object detectors \cite{dpm,fasterRCNN, RRC_detector} inspires great interests in the tracking-by-detection paradigm, in which objects in individual frames are detected first and then in the data association step these detections are linked into trajectories to recover their identities. There are two major components of a good data association method: a strong affinity model for similarity measure between detections across frames, and a robust identity inference model based on the similarity scores \cite{motsurvey2,motsurvey3}. For the affinity models, both hand-crafted design \cite{AB3DMOT, bydpixels} and deep learning-based design \cite{LSST17, deep_network} have made great progress and achieved promising discriminative power \cite{trackers_mot16}. For the identity inference, most of recent state-of-the-art trackers were based on maximum a posteriori (MAP) estimation \cite{motsurvey2, LSST17,AB3DMOT}, which provides a principled approach to optimally integrate information from detection reliability, missingness of detections, and affinity between detections. In the past ten years, one celebrated progress on the identity inference is the formulation of the MAP problems over multi frames as a minimum-cost flow problem \cite{mcf_kssp,mcf4mot, followme, deep_network}. This minimum-cost flow formulation enables naturally the incorporation of information from multiple frames, and more importantly it makes the MAP problems globally solvable with polynomial time complexity. For the identity inference model, one critical observation is that usually, the more frames are taken into account, the more pieces of evidence can be collected, and the more preferable inference results can be achieved \cite{followme}. In online tracking scenarios, as many previous frames as possible should be utilized in extending existing trajectories to the current frame; in offline scenarios, the inference should be carried out on the whole movie.

However, in real applications, minimum-cost flow-based identity inference models \cite{mcf_kssp,mcf4mot, followme} are limited to datasets with a small number of detections or frames due to computational cost. Despite extensive efforts of acceleration \cite{ssp4mot,followme}, the computational efficiency is still far from satisfactory. Therefore, compared with the wide application of advanced affinity models, the majority of recent works on MOT still rely on simple and greedy data association strategies. For example, some researchers sequentially conduct the inference over a small batch of the video \cite{followme} or only two consecutive frames \cite{LSST17, bydpixels} in the online scenarios with, \textit{e.g.}, the Hungarian algorithm. Some use greedy dynamic programming approximation, searching for sub-optimal linking results in the offline scenarios \cite{ssp4mot}.
The practical difficulty in utilizing a sufficiently long history of detections limits the performance of the identity inference. The whole tracking-by-detection paradigm is thus compromised and cannot achieve its full potential \cite{ssp4mot,followme}, which is corroborated in our experiments.

\begin{table*}[]
\centering
\caption{Comparison of Computational Complexity Bounds between Our Method and Existing Works}
\begin{threeparttable}
\small
\begin{tabular}{llll}
\hline
Method & Year & Bound                                            & Reference \\ \hline
MCF & 2008 & $O(n^2m\text{log}(nC)\text{log}(n))$             & \cite{mcf4mot}       \\
KSP & 2011 & $O(n(n\text{log}(n)+m))$                         & \cite{mcf_kssp}       \\
SSP & 2011 & $O(n(n\text{log}(n)+m))$                        & \cite{ssp4mot}       \\
dSSP & 2015 & $O(n(n\text{log}(n)+m))$                        & \cite{followme}       \\
Block-CS & 2017 & $O(\text{min}\{n^{2/3}, m^{1/2}\}m\text{log}(nC)\text{log}(n))$             & \cite{unitCapMCC}       \\
CINDA & present & $O(n^{1/2}m\text{log}(nC))$                      &  This work     \\\hline
\end{tabular}
\begin{tablenotes}
      \footnotesize
      \item \textit{* $m, n, C$: graph properties which are proportional to the number of detections, the number of potential links between detections and the maximum absolute value of similarity scores.}
\end{tablenotes}
\end{threeparttable}
\label{table:polyalg4mcf}
\end{table*}

In this report, we propose to formulate the MAP problems under the minimum-cost circulation framework, instead of the widely used minimum-cost flow framework. We name our method CINDA, which stands for CIrculation Network-based Data Association. CINDA can directly solve the MAP problem, while methods using the minimum-cost flow alone cannot. 
This is because the minimum-cost flow framework, as used in \cite{mcf4mot, deep_network, mcf_kssp, followme, ssp4mot}, assumes the optimal amount of flow in the network is predetermined. 
Yet, the optimal flow amount is equal to the number of objects we are tracking, which itself is unknown beforehand under most circumstances. 
Therefore, a great deal of additional effort needs to be spent on searching for the optimal flow amount. Some methods use a binary search scheme \cite{mcf4mot, deep_network}, where, in each step, a minimum-cost flow problem has to be solved from scratch, because the number of objects assumed in this round of searching can be very different from the previous one. 
Some other methods such as \cite{mcf_kssp, ssp4mot}, and \cite{followme} use incremental strategy to search for the optimal object number. Each iteration of the searching increases the amount of flow by one and requires traversing the whole network for finding a shortest path. 
Therefore, when the number of object is large, the cost for the incremental searching is high and these methods still suffer from the efficiency problem of the minimum-cost flow framework. 
Importantly, we observe that these efforts of searching for optimal object number can be saved by reformulating the MAP problem as a minimum-cost circulation problem instead, where the object number does not need to be explicitly searched binarily or one by one and is actually a natural byproduct of the integrated optimization procedure.

Taking advantage of the recent progress in network flow algorithms and exploring the special properties of the associated graph, we designed a  minimum-cost circulation algorithm in CINDA, which achieves the best ever theoretical complexity for solving the global data association problem.
Currently, the best theoretical complexity in solving the data association problem in MOT came from the blocking flow-based cost-scaling algorithm \cite{unitCapMCC}. Recently, researchers found that the minimum-cost flow problem in MOT can be solved by this algorithm in $O(\text{min}\{n^{2/3}, m^{1/2}\}m\text{log}(nC))$.
Here $m$, $n$, and $C$ are graph properties that are proportional to the number of detections, the number of potential links between detections, and the maximum absolute value of similarity score, respectively (see section 4 for details).
Thus, the global data association can be solved in $O(\text{min}\{n^{2/3}, m^{1/2}\}m\text{log}(nC)\log(n))$ ('Block-CS' in Table \ref{table:polyalg4mcf}). It is possible to use the same blocking flow-based cost scaling algorithm to solve the minimum-cost circulation problem. Unfortunately, it is widely known that that algorithm is less efficient in real applications \cite{sp_scaling,goldberg1997cs2}. 
To achieve a better theoretical complexity while keeping the real world efficiency in mind, we start our algorithm development from the famous push/relabel-based cost-scaling algorithm cs2 \cite{goldberg1997cs2}, which is practically efficient but is proven to have a worse theoretical complexity bound of $O(nm\log(nC))$.
We found that its complexity bound is worse because the flow pushing steps are not designed to correct the flow imbalance across nodes. As a consequence, some iterations may make no progress on reducing the flow imbalances in the network at all. 
Inspired by the blocking flow-based cost-scaling algorithm \cite{unitCapMCC}, in CINDA, the flow pushing operations are instead guided by a blocking flow, which dramatically improves the complexity to $O(\text{min}\{n^{2/3}, m^{1/2}\}m\text{log}(nC))$.
Notice that we already obtained $\log(n)$ folds improvement comparing with the best bound of minimum-cost flow-based frameworks because we get rid of the binary search scheme for the optimal object number.
Moreover, we identified a special property of the graph that further influences the complexity of the algorithm: an overwhelming majority of vertices in the graphs formulated in our framework are with unit capacity. That is, almost all vertices in our graph have only one input arc or one output arc. Taking this special property into account, we proved that the worst case complexity bound of CINDA can be further improved to $O(n^{1/2}m\text{log}(nC))$.

Consistent with the best theoretical complexity bound, CINDA achieves the best empirical computational performance through extensive experiments on five MOT benchmarks with various object detectors and affinity models. Among all these 40 experiments, CINDA is always the fastest compared with minimum-cost flow-based methods. 
Averagely, our method is 29 times faster than the best competing method for each experiment. Compared with each individual method, CINDA is 1,192 and 592 times faster than the successive shortest path algorithm (SSP) and its improved version, dSSP \cite{followme}, respectively. CINDA is 53 times faster than MCF \cite{mcf4mot}, which is based on cs2 \cite{goldberg1997cs2}, the famous implementation of cost-scaling algorithm. In addition, we test the scenario that our circulation-based framework serves a sub-module for more sophisticated global data association methods with higher-order constraints \cite{chari2015pairwise}, and similar running time improvement is attained. 

Besides saving time directly, the efficiency preponderance brings two further benefits. First, it makes possible for the inference model to utilize many more history frames in real-time applications and readily use the whole movie for offline applications. 
Second, it naturally enables us to build iterative tracking frameworks to refine affinity scores and increase their discriminative power by designing more features using the tracking results from previous iterations. 
We used two cases to demonstrate the accuracy improvement brought by the two benefits respectively. The first case is the car tracking scenario. The results show that even in the presence of state-of-the-art object detection results and affinity design, using longer history improves the overall MOT performance. This shows the three key factors in the tracking-by-detection paradigm (detector, affinity model, and global inference model) should work together to obtain the best result. The current practice of using two frames for identity inference gives sub-optimal results. Thankfully, our fast algorithm enables the consideration of more frames in computationally constrained and real-time applications. The other case is the cell tracking scenario. As cells have only limited appearance features, the affinity score that purely based on the pairs of detections shows low discriminative power. With an iterative framework, more features, such as velocity information, can be estimated from previous iterations, and experiments show the iterative processing helps to improve the tracking accuracy significantly.

Our major contributions can be summarized as the following three points. (1) We proposed a new minimum-cost circulation-based framework for solving the MAP problem in MOT. (2) We developed a minimum-cost circulation algorithm and proved that it achieves the best ever theoretical bound for solving the MAP inference problem in MOT. (3) We implemented our proposed algorithm, which was shown to be empirically much more efficient than existing widely used methods in all of the 40 benchmark tests, often by several orders of improvement. It is worth noting that the coincidence of the superiority in both theory and practice does not happen frequently for algorithm development because theoretical complexity is often related to the worst case performance, while practical efficiency involves many other factors. Such a coincidence assures users good practical performance without worrying about the worst scenarios. The source code of our implementation is publicly available on GitHub (\url{https://github.com/yu-lab-vt/CINDA}). Both Python and MATLAB interfaces are provided.

The remainder of the paper is organized as follows. In section 2 we review related works of identity inference models in MOT. The general form of the MAP identity inference in MOT is given in section 3. In section 4 we present the core methodology of our minimum-cost circulation-based framework. We give the proof of its complexity bound in section 5. We evaluate the efficiency and demonstrate several applications of our method in section 6 and conclude the paper in section 7.

\section{Related works}
Given the detections of objects in each frame of the video, the goal of the multi-object tracking problem is to select and cluster the detections corresponding to the same object over time \cite{chari2015pairwise}.
The selection and clustering task is commonly referred to as the data association problem, which involves two major components \cite{motsurvey2}:
one is the design of the affinity models to measure the similarity between detections, and the other one is the inference of the identities of detections based on the affinity measurements.
Based on the discriminative features learned or designed from the appearance cues of detections, there has been promising progress on building robust affinity models \cite{mot16, LSST17, bydpixels}. 
As another key component, recent state-of-the-art trackers conduct their identity inference under the MAP principle, where the formulation of minimum-cost flow has enjoyed most of the popularity \cite{fusionHeadQuadratic, motsurvey1, motsurvey2}. 
However, the current efficiency of solving the minimum-cost flow problem in MOT is far from satisfactory. Therefore, in large datasets or time critical scenarios like online tracking, globally optimal solutions gave way to simple greedy ones \cite{bydpixels, onlinePds}, which arguably limited the overall tracking performance \cite{bydpixels, motsurvey2}. 
In the following, we will first give a review to the development of recent identity inference models used in MOT and then analyze in details those minimum-cost flow-based ones. Finally, we will discuss the development of the theoretical complexity bound for solving the minimum-cost flow problem in MOT.

\subsection{Identity inference models}
MOT identity inference strategies can be divided into two categories: probabilistic approaches and deterministic approaches \cite{followme, motsurvey2}. Early probabilistic works like Multi-Hypothesis Testing \cite{MHT_org} and Joint Probabilistic Data-Association Filters \cite{JPDA_org}, have high computational cost and scale badly with the number of detections. Afterwards, plenty of approximation methods are proposed to accelerate these methods (see \cite{motsurvey2} for a good survey).
In recent years, the majority of the works model the identity inference as a deterministic MAP problem, which consists of three major ways of modeling: conditional random field (CRF) \cite{NOMT_CRF,disc_conti_CRF}, quadratic programming \cite{fusionHeadQuadratic, lagrangian2}, and minimum-cost flow \cite{mcf4mot, mcf_kssp, everyNeedsSome, LSST17}. Among them, minimum-cost flow-based methods (including special cases like k-shortest path and bipartite matching) are the most popular one thanks to its flexibility and the guarantee of optimality \cite{motsurvey3, motsurvey2}. Many quadratic programming formulations \cite{fusionHeadQuadratic, lagrangian1, lagrangian2} also use minimum-cost flow algorithms as a sub-routine to get approximate solutions.
The CRF model, as a class of Markov random field models, has also been used extensively in MOT \cite{motsurvey3}. However, CRF suffers from its intractable inference and thus has no guarantee to global optimality \cite{CRF_approx_TIP, CRF_dpl_arxiv}. Deep learning plays an important role in learning the discriminative affinity scores between detections. Though recently its applications to the identity inference problem are also emerging \cite{motsurvey3, trackers_mot16,quadCNN}, they are beyond the scope of this study.

\subsection{Minimum-cost flow models}
The minimum-cost flow-based formulation in identity inference problems first appeared in \cite{mcf4mot} and \cite{mcf_kssp}.
This formulation is flexible for dealing with detection flaws such as false positives and false negatives by appropriately designing arcs in the networks. It can also easily integrate various affinity designs. For the original formulation in the seminal works, the affinity design was only based on distances between detections, while other discriminative features like those from object's appearance were ignored. These features are frequently used and play an important role in more recent works \cite{mcfAppearance, followme}. 
Besides, it is also possible to do end-to-end learning of cost functions for the arcs in the network without hand-crafted arc affinity design \cite{deep_network}.
Some researchers \cite{chari2015pairwise, everyNeedsSome, lagrangian1, lagrangian2} designed networks considering higher-order constrains like consistency of velocity or restrictions from interacting objects or object parts. These additional constraints usually make the problem NP-hard. Therefore, these methods use Frank-Wolfe algorithm, Lagrangian relaxation, or other iterative solvers to get an approximate solution. Minimum-cost flow is commonly used to serve as a sub-routine in those iterative approximation methods.
Comprehensive surveys can be found in \cite{motsurvey1, motsurvey2, motsurvey3}. Notice that the same approach is called k-shortest paths in some papers.

Though there are many existing solvers for generic minimum-cost flow problem, directly applying them to identity inference models in MOT yields sub-optimal efficiency.
\cite{mcf_kssp} and \cite{ssp4mot} proposed that successive shortest path (SSP) algorithm has a low complexity bound for solving the minimum-cost flow problem in MOT when the number of targets is small. \cite{followme} showed that using the dynamic shortest path tree, SSP can be further accelerated, though the acceleration is limited. In practice, even though these solvers are specifically selected or designed for MOT problem, none of them is satisfactory. Hence in many real applications people tended to conduct inference on smaller batch of frames, especially in some online tasks.
For example, in \cite{followme}, the authors proposed to apply minimum-cost flow on a small time window of the video each time. For many recent online trackers \cite{bydpixels, LSST17, JBNOT, AB3DMOT}, they only considered optimizing linkages between two adjacent frames. For these methods, the minimum-cost flow problem is equivalent to a bipartite matching problem, which can also be solved by the Hungarian algorithm. These local optimal solutions usually cannot correct errors occurred in earlier frames when detections in later frame become available. As a result, they have accuracy loss compared with global approaches \cite{ssp4mot, followme, bydpixels}. This kind of accuracy loss has been recorded and efforts were tried to recover it through iterative or post processing. For instance, to deal with this problem, \cite{LSST17} used an iterative framework to correct errors from current greedy linking results, while \cite{AB3DMOT} added depth information and used 3D Kalman filter to reduce the identity ambiguity.

\subsection{Theoretical complexity of solving minimum-cost flow in MOT}
After the introduction of minimum-cost flow framework to the MOT problem, plenty of works  discussed the theoretical complexity for solving it \cite{mcf4mot, mcf_kssp, ssp4mot}. The representative results are summarized in Table \ref{table:polyalg4mcf}.
The cost-scaling algorithm cs2 \cite{goldberg1997cs2}, which is widely viewed as the best solver for generic minimum-cost flow problem, was utilized in the seminal work MCF \cite{mcf4mot}, whose worst-case complexity was shown to be $O(n^2m\text{log}(nC)\text{log}(n))$. Here $O(n^2m\text{log}(nC))$ is the complexity of solving a generic minimum-cost flow problem and $\text{log}(n)$ represents the iterations needed for finding the best flow amount.
Another type of algorithms, including the k-shortest path algorithm (KSP) \cite{mcf_kssp} or the successive shortest path (SSP) \cite{ssp4mot} algorithm, solve the corresponding minimum-cost flow problem by incrementally approaching to the optimal flow number.
Notice that the flow networks formulated in MOT problems are unit-capacity graphs, because it is assumed that each detection can participate at most one trajectory. Taking advantage of this fact, the step that increases the flow number, which is equivalent to instantiate a new trajectory, has lower computational cost. Therefore, SSP and KSP enjoy a complexity of $O(n(n\text{log}(n)+m))$, which is lower than the complexity of cost-scaling algorithms.
If we assume the number of targets is $K$, the complexity can be written as $O(K(n\text{log}(n)+m))$. Besides, \cite{followme} accelerated existing SSP algorithm by leveraging the idea of dynamic shortest path tree (dSSP in Table \ref{table:polyalg4mcf}). Recently, researchers in the graph theory community found that the complexity of some previously proposed cost-scaling algorithms are actually lower when they are applied on unit-capacity graphs \cite{unitCapMCF,unitCapMCC}. Thus, applying the new theory, the complexity can be further decreased to $O(\text{min}\{n^{2/3},m^{1/2}\}m\text{log}(nC)\text{log}(n))$ using blocking flow-based cost-scaling algorithm to solve the minimum-cost flow problem in MOT . Note that to the best of our knowledge, this complexity has never been reported in MOT literature.

In this paper, we carefully analyzed the MAP problem and proposed to formulate it into a minimum-cost circulation framework rather than a minimum-cost flow framework. This new formulation significantly improves the theoretical worst case complexity bound of solving the MAP problem and brings orders of efficiency boosting in the experiments. In the following sections, the term 'object' represents a physical object existing over time (\textit{e.g.} a person), and 'detection' indicates a detected snapshot of an object at a given time point. An object corresponds to a series of detections (trajectory) with the same identity.

\begin{figure*}[t]
\centering
\includegraphics[width=1\linewidth]{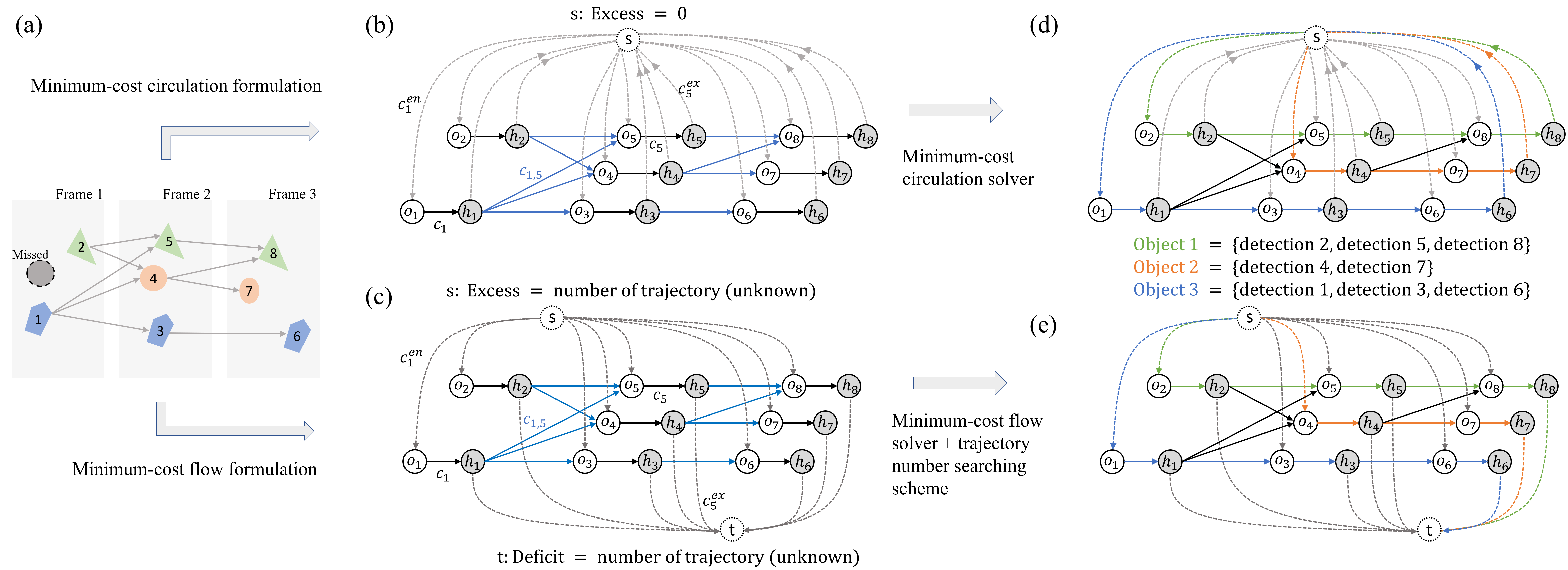}
\caption{
(a) Objects detected in three consecutive frames. The first frame contains two detections; one missed detection is colored in gray. Lines between detections are the possible ways of linking them. Each line is associated with a cost. 
If the similarity between two detections is too low to be the same object, we do not link them. There are three trajectories in these three frames. For example, detections 1, 3 and 6 should be linked together as a single trajectory.
(b) The proposed minimum-cost circulation formulation for MOT problem. Detection $x_i$ is represented by a pair of nodes: a pre-node $o_i$ and a post-node $h_i$. The dummy node $s$ is linked to all pre-nodes. 
Then all post-nodes are linked back to $s$. These edges are shown in dashed lines. Transition edges between detections are shown in blue. Similar to (a), there is no transition edge if the similarity between detections is too low. The input and output flows of every node in the circulation network is balanced. Therefore, the excess for the dummy node is always 0. 
(c) Typical minimum-cost flow formulation for MOT problem. Each detection is also represented by a pre-node and post-node. The difference is that in this flow network, there are two dummy nodes: the dummy source node $s$ is linked to all pre-nodes $\{o_i\}$  and the dummy sink node $t$ is linked to the post-nodes $\{h_i\}$. The color coding is the same as that in (b). The input and output flows of these two dummy nodes are both imbalanced and the amount of imbalances decide how many targets we want to track (the amount of flow that can happen in the network). To apply existing minimum-cost flow solvers, we need to specify the excess and deficit of $s$ and $t$ first, which is commonly unknown. 
(d) The results from the proposed minimum-cost circulation framework. Three trajectories are created and they are shown with the same color as in (a). The solution can be obtained by any minimum-cost circulation solver. 
(e) The results from the minimum-cost flow formulation. The same three trajectories are generated.
In addition to a minimum-cost flow solver, an accompanying searching scheme is needed to find the optimal trajectory number, or equivalently, the optimal flow amount.
}
\label{fig:mcf}
\end{figure*}

\section{Problem statement}
The problem of identity inference can be formulated as an MAP problem. Here we discuss the most widely used formulations that consider only unary and pairwise relationships \cite{mcf4mot, mcf_kssp, followme}. This form is also frequently employed as the sub-routine in quadratic formulations where higher order constraints are considered \cite{fusionHeadQuadratic, everyNeedsSome, lagrangian1,lagrangian2}.

Let $\mathcal{X} = \{x_i\}$ be a set of detections, where $x_i$ is a vector containing the position, appearance, and time index of detection $i$. A trajectory $T_k = \{x_{k_1},\ldots, x_{k_n}\}$, $x_{k_i} \in \mathcal{X}$ is a set of temporally ordered detections and 
\begin{align}
T_k \cap T_l = \emptyset, \forall k \neq l.
\end{align}
An association hypothesis is a set of non-overlap trajectories $\mathcal{T}=\{T_k\}$. The purpose of identity inference is to find the hypothesis $\mathcal{T}$ with the highest posterior probability \cite{mcf4mot}:
\begin{align}
\mathcal{T}^* &= \argmax_\mathcal{T}P(\mathcal{T}|\mathcal{X})\\
&=\argmax_\mathcal{T}P(\mathcal{X}|\mathcal{T})P(\mathcal{T})\\
&=\argmax_\mathcal{T} \prod_i P(x_i|\mathcal{T})\prod_{k:T_k\in \mathcal{T}} P(T_k).
\end{align}
Assume $P(x_i|\mathcal{T})$ follows a Bernoulli distribution. The corresponding parameter $\beta_i$ indicates the probability that detection $x_i$ is false positive and thus should be excluded in a trajectory in $\mathcal{T}$:
\begin{align}
P(x_i|\mathcal{T})=\left\{
\begin{array}{ll}
1-\beta_i,  \hspace{3mm} x_i\in T_k, T_k\in \mathcal{T} \\
\beta_i,\hspace{9mm}  otherwise.
\end{array}
\right.
\end{align}
Since we only consider unary and pairwise relationships between detections, a trajectory $T_k=\{x_{k_1},\ldots, x_{k_n}\}$ can be modeled as a Markov chain whose probability is
\begin{align}
P(T_k) = P_{enter}(x_{k_1})\prod^{n-1}_{i=1} P(x_{k_{i+1}}|x_{k_i})P_{exit}(x_{k_n}).
\end{align}
$P_{enter}(x_{k_1})$ is the probability that $x_{k_1}$ is the initial point of trajectory $T_k$. Similarly, $P_{exit}(x_{k_n})$ is the probability that $x_{k_n}$ is the terminate point of trajectory $T_k$. By taking the negative logarithm of all the probabilities in Eq.(4), the MAP problem can be converted to an integer linear programming (ILP) problem \cite{mcf4mot}:
\begin{align}
          f^* = \argmin_f\  &\sum_i C_i f_i + \sum_i C^{en}_i f^{en}_i  + \sum_{i,j}C_{i,j}f_{i,j} \nonumber \\
          + \sum_i C^{ex}_i f^{ex}_i \\ 
\text{s.t.} \ \ \ &f_i, f^{en}_i, f^{ex}_i, f_{i,j} \in{\{0,1}\} \\
\text{and} \ \ \ &f^{en}_i + \sum_j f_{j,i} =f_i= f^{ex}_i + \sum_j f_{i,j}
\end{align}
with
\begin{align}
&C^{en}_i = -\log{P_{enter}(x_{i})}, \hspace{5mm} C^{ex}_i = -\log{P_{exit}(x_{i})} \\
&C_{i,j} = -\log{P(x_{i}|x_{j})},  \hspace{7mm}C_i=\log{\frac{\beta_i}{1-\beta_i}}. %
\end{align}
Constraints in Eq.(8-9) indicate that each detection can participate in at most one trajectory and there is no splitting or merging of any trajectory.
More specifically, as $f_i \in \{0,1\}$, $P(x_i|\mathcal{T})$ can be rewritten as $(1-\beta_i)^{f_i}\beta_i^{(1-f_i)}$. $f_i = 1$ indicates that detection $x_i$ is included in a trajectory of $\mathcal{T}$ and $f_i = 0$ otherwise. 
Based on the other constraints in Eq.(8-9), we can see $f_i=1$ forces detection $x_i$ to be incident with at most one previous detection and at most one following detection in a trajectory, so $x_i$ will participate one and only one trajectory. 
Under such condition, $f^{en}_{i}=1$ or $f^{ex}_i=1$ indicates that $x_i$ is the initial or terminate point of a trajectory in $\mathcal{T}$. $f_{i,j}=1$ means that detection $x_i$ is followed by  detection $x_j$ in the same trajectory in $\mathcal{T}$. 
$f_i=0$ automatically rules out the detection $x_i$ from participating in any trajectory, where $f^{en}_i, f_{*,i}, f^{ex}_i$, and $f_{i,*}$ will all be zeros. 

\section{Minimum-cost circulation framework \label{method:mcc}}
We propose to map the ILP problem in Eq.(7-9) into a minimum-cost circulation problem for efficient inference. In this section, we will detail how the mapping is performed and prove their equivalence. For simplicity, our network will be named \textit{circulation network} (\textit{e.g.}, Fig. \ref{fig:mcf}(b)) while any minimum-cost flow-based network (\textit{e.g.}, Fig. \ref{fig:mcf}(c)) used in previous works will be called \textit{flow network}. The whole pipeline of our framework can be found in Fig. \ref{fig:flowchart} and more details are given in Alg. \ref{alg:cs2}.

We denote a circulation network as $G(V,E)$ (\textit{e.g.}, Fig. \ref{fig:mcf}(b)), whose node set is $V$ and arc set is $E$. The number of nodes is $n$ and the number of arcs is $m$. Each arc $(v,w)\in E$ is associated with unit capacity $u(v,w)$ and a real-valued cost $c(v,w)$. The graph has a dummy node $s$. For each detection $x_i\in \mathcal{X}$, two nodes, a pre-node $o_i$ and a post-node $h_i$, and three arcs, $(o_i,h_i)$, $(s,o_i)$ and $(h_i,s)$ are created. The corresponding costs are $C_i$, $C_i^{en}$, and $C_i^{ex}$, respectively. A transition arc is then created for any allowed spatiotemporal transition of an object. For example, if it is allowed for detection $x_i$ to transit to detection $x_j$ (assume $x_i$ is detected before $x_j$ in time), an arc $(h_i,o_j)$ is created with the cost of $C_{i,j}$. An example is given in Fig.\ref{fig:mcf}(b), which is constructed based on the detections depicted in Fig.\ref{fig:mcf}(a). The capacity of any arc is set to one.

Now we prove the solution of minimum-cost circulation formulation is the same as the MAP problem in Eq.(4). Since minimum-cost flow formulation was shown to have the same optimal solution as the MAP problem \cite{mcf4mot}, if we can prove the equivalence between our formulation and minimum-cost flow formulation, we can thus claim that our formulation also solves the MAP problem. We first prove a lemma that will help us later to link the solution of  minimum-cost circulation to minimum-cost flow.

\begin{lemma}\label{lemma:cycle}
The dummy node $s$ is included in every directed cycle in the circulation network.
\end{lemma}
\begin{proof} 
Assume the circulation network is $G$. 
Let $G'=G\backslash s$ be the sub-graph of $G$, which is obtained from $G$ by pruning node $s$ and all arcs that are incident with $s$. We will prove that $G'$ is a directed acyclic graph (DAG). Since a DAG does not contain any cycle, every cycle will have to include the node $s$ that was removed.

It is clear that nodes in $G'$ come in pairs: each pre-node is accompanied by a post-node. Each pre-node has only one out-going arc whose head is a post-node and each post-node has only one in-coming arc whose tail is a pre-node. 
Post-nodes can only link to pre-nodes, which we call the linkages transition arcs. Assume there is a cycle in $G'$, and without loss of generality we denote it as $\{((o_1,h_1), (h_1,o_2), \cdots, (o_n,h_n), (h_n,o_1))\}$. 
Based on the construction of the graph, each transition arc $(h_i,o_j)$ corresponds to a pair of detections $x_i$ and $x_j$, and $x_i$ is detected strictly earlier than $x_j$. Since $h_n$ links to $o_1$, $x_n$ should happen earlier than $x_1$. By repeatedly applying the same reasoning for detections $\{x_1,\dots,x_n\}$, we know that $x_1$ should happen earlier than $x_n$ and this contradicts to our previous statement that $x_n$ should happen earlier than $x_1$. Therefore, $G'$ is DAG and any cycle in $G$ goes through $s$.
\end{proof}

Under such construction, we can interpret each cycle in $G$ as an object trajectory candidate, linking a sequence of detections, starting from $s$ and ending at $s$ as well.
Since we define the cost function $c(v,w)$ as $c(o_i,h_i)=C_i$, $c(s,o_i)=C^{en}_i$, $c(h_i,s)=C^{ex}_i$, and $c(o_i,h_j)=C_{i,j}$, the MAP problem in Eq.(7-9) could be solved by selecting a set of object trajectories from all the candidates so that they lead to minimum total cost in the graph. Algorithmically, This is achieved by sending circular flows in $G$. When the algorithm terminates, the cycles with flow inside them are selected as the output trajectories.
As in-out balance is required at any node $v\in V$ (conservation constraints in Eq.(9)), we have:
\begin{lemma}\label{lemma:k_flow}
A circulation with total integer flow amount $K$ can only be sent through $K$ distinctive cycles.
\end{lemma}
\begin{proof}
Because the capacity of each arc is one, by sending a flow with amount one along a cycle, all arcs of the cycle are saturated. Thus a circulation with flow amount $K$ saturates $K$ cycles with no shared arcs.
\end{proof}

Lemma \ref{lemma:k_flow} shows that the minimum-cost circulation formulation could detect non-overlapping object trajectories, as required in the MAP formulation.

The key difference between our circulation network and previously widely used flow network (Fig. \ref{fig:mcf}(c)) is that circulation network does not have excess or deficit nodes (imbalanced nodes). In the flow network, the entrance and exit probabilities are encoded using an excess node $s$ and an deficit node $t$ as shown in Fig. \ref{fig:mcf}(c). A value of imbalance needs to be specified for each of them before applying the solver. This imbalance is the expected trajectory number $K$, or flow amount equivalently, which is unknown and additional efforts are needed to optimize it. In the circulation network, this is avoided because all nodes are balanced. By pursuing the minimum cost in the circulation network, the optimal number of object trajectories will be simultaneously selected.

Now we prove the equivalence of our minimum-cost circulation problem and the MAP problem.

\begin{theorem}
The MAP problem in Eq.(7-9) is equivalent to a minimum-cost circulation problem.
\end{theorem}
\begin{proof}
Assume the circulation network is $G_c$ and the corresponding flow network is $G_f$ (see Fig.\ref{fig:mcf}(b-c)). Based on Lemma \ref{lemma:cycle}, for each cycle in $G_c$, by separating its dummy node, it can be translated to a unique $s$-$t$ path in $G_f$. By the same token, each $s$-$t$ path in $G_f$ can be translated to a unique cycle in $G_c$. Then with Lemma \ref{lemma:k_flow}, any feasible circulation with amount $K$ on $G_c$ corresponds to a feasible flow on $G_f$ with $K$ distinctive $s$-$t$ paths, and vice versa.
As is shown in \cite{mcf4mot}, the MAP problem in Eq.(7-9) is equivalent to finding a set of distinctive $s$-$t$ paths with minimum total cost on $G_f$, which is thus equivalent to finding the minimum-cost circulation on $G_c$.
\end{proof}

If the costs for all arcs in the circulation network are non-negative, the optimal solution is trivial: no cycle is detected, the value of the flow is zero, and the minimum cost is also zero. Because we have negative-cost arcs (Eq.(10-11)), some cycles in the network could have negative cost and the total cost can be negative.

It is worthy to mention that all the minimum-cost flow problems we encountered in MOT applications can be replaced by our minimum-cost circulation formulation, which means that all the existing minimum-cost flow-based identity inference models in MOT can be accelerated with the proposed formulation.
\begin{figure}[t]
\centering
\includegraphics[width=1\linewidth]{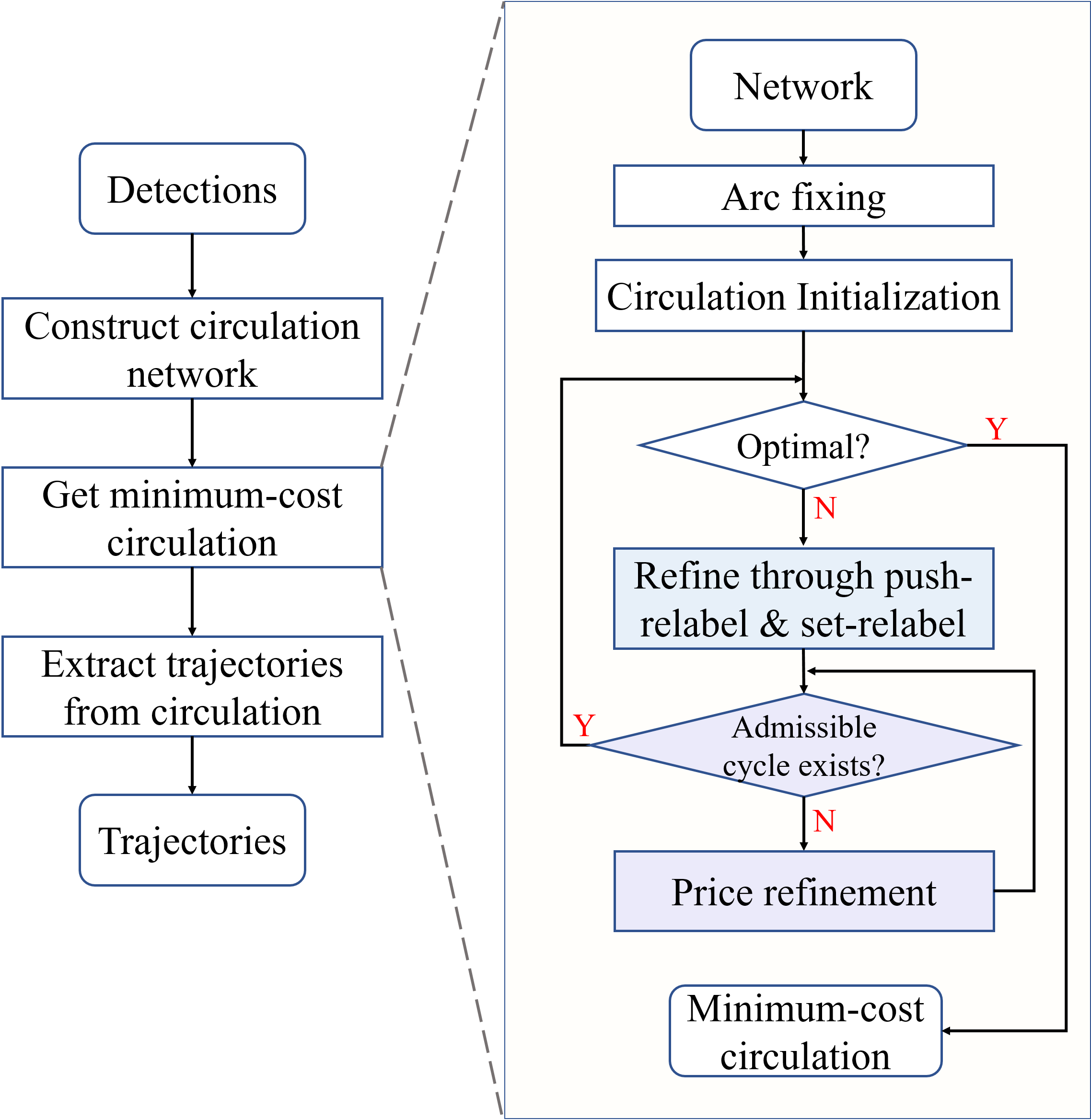}
\caption{
Flow chart of CINDA. The two modules shown with shaded colors take over the major time consumption of the whole framework. The pseudo code of these two modules can be found in the Alg. \ref{alg:cs2} (line 13-29).
}
\label{fig:flowchart}
\end{figure}

\section{Algorithm and worst case complexity}
The design of our minimum-cost circulation algorithm in CINDA is adapted from the famous push/relabel-based cost-scaling algorithm cs2 \cite{goldberg1997cs2}, which runs $O(nm\log(nC))$ for solving the minimum-cost circulation problem in MOT. 
Recent works \cite{unitCapMCC,unitCapMCF} proved that the blocking flow-based cost-scaling algorithm achieves the best ever bound $O(\text{min}\{n^{2/3}, m^{1/2}\}m\text{log}(nC))$ for solving minimum-cost circulation problem on general unit-capacity graphs. Enlightened by their works, we introduce the idea of blocking flow to our algorithm and achieve a theoretical bound of $O(n^{1/2}m\text{log}(nC))$. This bound is not only significantly better than cs2, but also better than the best bound mentioned above. The improvement compared with blocking flow-based cost-scaling algorithm is based on our observation that our graph is more than a unit-capacity graph: the overwhelming majority of the vertices in our graph has only one input arc or one output arc. More details will be given in section 5.3.

In the following, we will present the outline of CINDA implementation to show how the idea of blocking flow is aggregated and prove its complexity by leveraging these added ingredients and our own graph specialty. The outline of CINDA is given in Fig. \ref{fig:flowchart} and Alg. \ref{alg:cs2}. 
To begin with, we need several definitions, most of which are adapted from \cite{unitCapMCC}.

\subsection{Some definitions}
As in section \ref{method:mcc}, $G(V,E)$ is the input circulation network and we assign a cost function $c(v,w)$ and a unit capacity ($u(v,w)=1$) to each arc $(v,w)\in E$. For the purpose of analysis, for each arc $(v,w)\in E$, we add a reverse arc $(w,v)$ to $E$. The capacity of that arc is zero and the cost is $-c(v,w)$. A \textit{flow} $f$ on $G$ is an anti-symmetric function on arcs, \textit{i.e.}, $f(v,w)=-f(w,v)$. We call a flow \textit{feasible flow} if it satisfies all the capacity constrains for the arcs in $E$ (Eq.(8)) and conservation constrains for the nodes in $V$ (Eq.(9)). If only capacity constraints are satisfied, we call it a \textit{pseudo flow}. A \textit{circulation} is a special type of flow which can be fully decomposed into flows along cycles. On our circulation network, any feasible flow is a circulation.

Given a flow $f$, the \textit{residual capacity} of an arc $(v,w)$ is $r_{vw}=u(v,w)-f(v,w)$. An arc $(v,w)$ is called \textit{residual} if $r_{vw}>0$ and \textit{saturated} otherwise. 
The \textit{excess} of node $v$ with respect to flow $f$ is $e_f(v)=\sum_{(v,w)\in E}f(v,w)$. We call a node excess node if its excess is positive and deficit node if the excess is negative. The \textit{total excess} of a network is the summation of the excesses of all its excess nodes.

A \textit{cut} $X$ of $G$ is a non-empty proper subset of $V$. The \textit{excess} of $X$ is the summation of the excesses of vertices in $X$. An arc $(v,w)$ \textit{crosses} cut $X$ if $v\in X$ but $w\notin X$. 

We associate each node $v$ in $G$ with a real number $p(v)$. We refer $p(v)$ as the \textit{price} of node $v$. The \textit{reduced cost} of arc $(v,w)$ is defined as $c_p(v,w) = c(v,w) + p(v) - p(w)$. A flow $f$ is $\epsilon$-\textit{optimal} with respect to a price function $p$, if $c_p(v,w)\geq -\epsilon$, for each residual arc $(v,w)$. 

Arc $(v,w)$ is an \textit{admissible arc} if it is a residual arc with negative reduced cost. The \textit{admissible network} $G_A=(V_A,E_A)$ of $G$ is defined as the sub-network consisting solely of admissible arcs with capacity function $r_{vw}$ and cost function $c_p(v,w)$ and vertices incident with admissible arcs. 

    
     
    
\begin{algorithm}[t]
    \SetKwInput{KwInput}{Input}
    \SetKwInput{KwOutput}{Output}
    \SetKw{Continue}{continue}
    \DontPrintSemicolon
    
    \KwInput{Set of detections $\mathcal{X} = \{x_i\}$}
    \KwOutput{Set of trajectories $\mathcal{T} = \{T_k\}$}
    $G(V,E)\leftarrow \text{ContructCirculationNetwork}(\mathcal{X})$\\
    $G\leftarrow \text{Arc-fixing}(G)$\\
    $\epsilon \gets C$; $\textit{p(u)} \gets 0, \forall u\in V$ // \texttt{$C$:max arc cost} \\
    \lIf {$\exists$ \text{a feasible circulation } $f_0$}
    {$f\gets f_0$}
    \lElse{\Return null}
    \texttt{// main loop}\\
    \While {$\epsilon \geq 1/n$ }
    {
     $\epsilon\gets \epsilon/2$\\
      \texttt{// saturate admissible arcs}\\
     \For{each admissible arc $(v,w)$}{
        $f(v,w)\gets 1$
     }
     \texttt{// RESTORE through push\&set-relabel}\\
     \While {$f$ is not a feasible circulation}
     {
            \texttt{// step 1: set-relabel to create a blocking flow}\\
            \While {$\exists$ an excess node $v\notin V_A$ }{
                $S\gets\{v\in V_A | \text{ an deficit node is reachable from } v\text{ in } G_A\}$\\
                $\forall v\in S$, $p(v)\gets p(v)+\epsilon$ \\
            }
		 	\texttt{// step 2: push/relabel along the blocking flow}\\
        	\For{each excess node $v$ in $V$}{
            	push/relabel $v$ until it becomes balanced \\
            }
    }
         \texttt{// Decrease $\epsilon$ through price refinement}\\
     \While{\text{no cycle exists in $G_A$}}{
        \text{refine $p$ such that $f$ is $\epsilon/2$-optimal}\\
        \lIf{succeed}{$\epsilon\gets \epsilon/2$}
     }
    }
    $\mathcal{T} \leftarrow \text{Flow2Trajectories}(f)$ \\
    \Return $\mathcal{T}$
    
    \caption{
    CINDA: Circulation Network-based Data Association in MOT}
    \label{alg:cs2}
\end{algorithm}
Given a flow $f$ and its admissible network $G_A$, we call a flow $g$ on $G_A$ an \textit{improving flow}, if pushing $g$ on $G_A$ decreases the total excess of the network while does not create any new excess or deficit vertex. An improving flow is \textit{blocking} if after pushing $g$ on $G_A$ (with respect to flow $f$), no path in $G_A$ (with respect to flow $f+g$) from an excess vertex to a deficit vertex exists.

\subsection{Proposed minimum-cost circulation algorithm}
The major steps in our framework are shown in Fig. \ref{fig:flowchart} and Alg. \ref{alg:cs2}. The initialization steps construct the graph from detections. After clipping some large-cost arcs \cite{goldberg1997cs2}, we find an initial feasible circulation. 
In our implementation, the initial feasible circulation is obtained by simply setting the flow at each arc to 0. The main loop (line 7-30 in Alg. \ref{alg:cs2}) is built upon cost-scaling method for solving the minimum-cost circulation problem. Each iteration makes current circulation closer to the optimality, by refining an $\epsilon\text{-optimal}$ circulation to an $\epsilon/2\text{-optimal}$ circulation. Once $\epsilon < 1/n$, current circulation is the optimal solution \cite{networkbook}. The main difference between our algorithm and cs2 occurs in step 2 \textit{push/relabel} (line 21-23 of Alg. \ref{alg:cs2}). In our algorithm, the flow pushing is guided by the blocking flow directions, which is not the case in cs2. The implication of the modifications on the complexity analysis will be discussed in the next paragraph, as well as at the end of Section 5.4. Below we give further details of the algorithm.

In each iteration of the main loop, the threshold $\epsilon$ is halved and thus current circulation is no longer $\epsilon$-optimal. To refine it, all admissible arcs are first saturated so that we obtain an $\epsilon\text{-optimal}$ pseudo flow. However, what we want is a circulation where there is no excess or deficit node. Therefore the $\texttt{RESTORE}$ loop (line 14-24 in Alg.\ref{alg:cs2}) tries to restore an $\epsilon\text{-optimal}$ pseudo flow to an $\epsilon\text{-optimal}$ feasible flow, which is the circulation we want. There are two steps in $\texttt{RESTORE}$. The $\texttt{set-relabel}$ step updates the price function $p$ to generate a blocking flow from excess to deficit nodes so as to guide the flow pushing in the second step to decrease the total excess of the network. In each iteration of $\texttt{set-relabel}$, we will add $\epsilon$ to the prices of deficit nodes and nodes that have accesses to a deficit node in $G_A$ and thus create more admissible arcs, growing $G_A$ a little bit. $\texttt{set-relabel}$ stops when $G_A$ grows large enough to include all excess nodes. The $\texttt{push/relabel}$ step pushes flow out of excess nodes along the direction indicated by the blocking flow. In $\texttt{set-relabel}$, the price of each deficit node is guaranteed to increase $\epsilon$ in every iteration. In $\texttt{push/relabel}$, because the flow pushing is guided by blocking flow, we are guaranteed to decrease the total excess of the network at least one with $O(m)$ pushes. These two properties are the key for our proof of Lemma \ref{lemma:numIt}, which indicates that the bound to the number of iterations in one call of $\texttt{RESTORE}$ is equivalent to the bound of the total excess. After $\texttt{RESTORE}$, we get an circulation that is $\epsilon\text{-optimal}$. If there are no cycles in $G_A$ now, it is possible that current circulation is not only $\epsilon$-optimal but also $\epsilon/2\text{-optimal}$ \cite{goldberg1997cs2}. Under such condition, price refinement (line 26-29 in Alg.\ref{alg:cs2}) is conducted, trying to manipulate only the price function $p$ such that the circulation becomes $\epsilon/2\text{-optimal}$.

\subsection{Special structures of the graph in MOT identity inference problem}
A graph is called a \textit{unit-capacity graph}, if the capacity of every arc is one. The graphs of our formulated circulation networks in the MOT problem are unit-capacity graphs. 
A vertex is with \textit{unit capacity} if it is associated with only one incoming arc (no limitation to outgoing arcs) or one outgoing arc (no limitation to incoming arcs). 
The other specialty of the graphs in our circulation networks is that except for the dummy node $s$, each node is with unit capacity. 
For a unit-capacity graph, if all of its vertices are with unit capacity, we call it a \textit{unit-vertex-capacity graph}. If most of the vertices in the graph are with unit-vertex-capacity but with $O(1)$ exceptions, we call such graph as \textit{overwhelming unit-vertex-capacity graph}. Our graphs are overwhelming unit-vertex-capacity graphs with only one vertex exception. This property is critical for our proof of Lemma \ref{lemma:t_ex}.
In the following, we will prove that based on CINDA shown in Fig. \ref{fig:flowchart} and Alg. \ref{alg:cs2}, the minimum-cost circulation problem on our graphs along with the whole identity inference problem can be solved in $O(n^{1/2}m\text{log}(nC))$. 

\subsection{Proof of the complexity}
It is clear that other than the main loop (line 7-32 in Alg. \ref{alg:cs2}), the steps of constructing the circulation network from the detections (line 1), clipping high-cost arcs (line 2) and extracting trajectories from the minimum-cost circulation (line 31) have the overall complexity of $O(m)$. 
Thus, the bound of our framework is determined by the complexity of the main loop. Note that these clipped arcs will be checked occasionally in case they will still be utilized in the optimal solution. Because $\epsilon$ is halved at least one time in each iteration, the main loop terminates within $\log(nC)$ iterations \cite{networkbook}. Since the task of refining an $\epsilon/2$-optimal pseudo flow into an $\epsilon/2$-optimal feasible flow is fulfilled by the inner loop \texttt{RESTORE} (line 14-24 in Alg. \ref{alg:cs2}), we first prove the bound $O(n^{1/2}m)$ for one call of \texttt{RESTORE} as summarized in Lemma \ref{lemma:restore}.

From Alg. \ref{alg:cs2}, we can see that the prices of nodes are non-decreasing in step 1 of $\texttt{RESTORE}$. This is also true in step 2 when we relabel excess nodes. All the changes are in the unit of $\epsilon$. Thus, we assign each node a non-negative integer value $d$ to indicate how many $\epsilon$s have been added to its price. For any node $v\in V$, $d(v)=(p(v)-p_0(v))/\epsilon$. Here $p_0(v)$ is the initial price of node $v$ in the current call of $\texttt{RESTORE}$. Assume $f_0$ is the initial $\epsilon$-optimal flow in the current call of $\texttt{RESTORE}$ before we saturate admissible arcs and $f$ is the current $\epsilon/2$-optimal pseudo flow. Define the arc set $E^+=\{(v,w)\in E | f(v,w)>f_0(v,w)\}$ based on flow $f$. We divide the iterations in a single call of $\texttt{RESTORE}$ into two phases. The first phase ends when $d(v)\geq\triangle$ for any deficit node $v$ and the second phase finishes the restoration. Parameter $\triangle$ can be any positive real number, whose optimal value will be discussed in Lemma \ref{lemma:restore}. The following lemma outlines the total number of iterations we need for the two phases.

\begin{lemma}\label{lemma:numIt}
The first phase ends with $O(\triangle)$ iterations of \texttt{step 1} and \texttt{step 2}. 
Assume the total excess is $\tau$ when first phase ends. Then the second phase takes $O(\tau)$ iterations of \texttt{step 1} and \texttt{step 2}.
\end{lemma}
\begin{proof}
The first part of the lemma is clear that one iteration of step 1 will increase the price of every deficit node $\epsilon$, while step 2 will not create new deficit nodes or decrease the price of any node. The proof of the second part needs some implementation details. In the implementation of step 2, we do not wait until that all the excess nodes become balanced. Otherwise, the $\texttt{RESTORE}$ will be fulfilled with one iteration of step 2 with a high complexity of $O(nm)$. We will break the for loop after $O(m)$ pushes have been conducted. The pushes are guided by the blocking flow found in step 1. 
In more details, in step 2, the push/relabel operations will be conducted first on the excess nodes that the blocking flow goes through.
For each excess node, its out-going arcs that the blocking flow occupies has higher priority to push flow on. 
Thus with at most $m$ pushes, we will saturate one path from an excess node to a deficit node. Pushing flow along this saturated path decreases the total excess of the network one. 
Because the push/relabel operations in step 2 never increase the total excess and neither does step 1, so the second phase takes totally $O(\tau)$ iterations of \texttt{step 1} and \texttt{step 2}.
\end{proof}
From Lemma \ref{lemma:numIt}, we know that the total number of iterations in one call of \texttt{RESTORE} is $O(\triangle + \tau)$. Then the question is to find a bound of $\tau$ by finding the maximum number of total excess at the end of the first phase. We will show that this bound is related to $\Delta$ and we can get the overall bound combining phase 1 and phase 2 by selecting an optimal $\Delta$. To start with, we need the following two lemmas.

\begin{lemma}\label{lemma:d}
If $(v,w)\in E^+, \text{ then } d(w)\geq d(v)-3$.
\end{lemma}
\begin{proof}
As $(v,w)\in E^+$, $f(v,w)>f_0(v,w)$, so arc $(v,w)$ is a residual arc corresponding to flow $f_0$, and $(w,v)$ is a residual arc correspond to flow $f$. Since we have halved the $\epsilon$ at the beginning of this \texttt{RESTORE}, $f_0$ now is $2*\epsilon$-optimal and $f$ is $\epsilon$-optimal. Thus, $c_{p_0}(v,w)=c(v,w)+p_0(v)-p_0(w) \geq -2*\epsilon$ and $c_p(w,v)=c(w,v)+p(w)-p(v) \geq -\epsilon$. Summing these two inequalities together, we have $d(w) - d(v) \geq -3$.
\end{proof}
\begin{lemma}\label{lemma:excess}
Assume $X$ is a valid cut for flow $f$ on a unit-capacity graph. If $X$ contains all the deficit nodes while no excess nodes, the total excess of the graph is at most $|\{(v,w)\in E^+ | (v,w) \text{ crosses } X\}|$.
\end{lemma}
\begin{proof}
Let $\bar{X}=V-X$. For any arc crosses $X$, its inverse arc crosses $\bar{X}$. It is clear that $\bar{X}$ contains all the excess nodes in the graph while no deficit nodes. Thus, for $\bar{X}$, its excess is equal to the total excess of the graph. As $f$ and $f_0$ are anti-symmetric functions on arcs, 
$\sum_{v\in \bar{X}}e_f(v) = \sum \{f_0(v,w)-f(v,w)|(v,w) \text{ crosses }\bar{X}\} 
= \sum\{f(w,v) - f_0(w,v) | (w,v)\text{ crosses }X\}
\leq \sum \{f(w,v)-f_0(w,v) | (w,v) \text{ crosses } X \ \&\ (w,v)\in E^+ \}$. 
Because $f$ and $f_0$ are both binary functions, $\sum_{v\in \bar{X}}e_f(v) \leq |\{f(w,v)-f_0(w,v) | (w,v) \text{ crosses } X \ \&\ (w,v)\in E^+ \}|$. Thus the total excess is bounded by the number of arcs in $E^+$ that cross $X$.
\end{proof}
Now we prove the bound for $\tau$.
\begin{lemma} \label{lemma:t_ex}
The total excess when first phase ends is $O(n/\triangle)$.
\end{lemma} 

\begin{proof}
Our circulation network is an overwhelming unit-vertex-capacity graph where only the vertex $s$ is not with unit capacity. We can prove this lemma based on this specialty.
Consider the cuts $X_i=\{v|d(v)>i\ \&\ v \text { is not an excess node}\}$ for $0\leq i < \triangle$. We assume none of the $X_i$s is empty and $\triangle > 6$, otherwise the lemma is true based on Lemma \ref{lemma:excess}. From the definition of the first phase, we know all deficit nodes are included in any $X_i$ while no excess node is included. 
For $3\leq i < \triangle-3$, let $Y_i=X_{i-3}-X_{i+3}$. Any vertex can be in at most six sets of $Y_i$ and thus there exists some $Y_i$ that contains at most $6n/(\triangle-6)$ vertices. 
Let $Z'_i=X_i-X_{i+3}$ and  $Z''_i=X_{i-3}-X_i$. It is clear that $Z'_i \subseteq Y_i$ and $Z''_i \subseteq Y_i$, so $|Z'_i|\leq 6n/(\triangle-6)$ and $|Z''_i|\leq 6n/(\triangle-6)$ for some $i$. 
Assume every node has either only one incoming arc or one outgoing arc, except node $s$. Therefore, $s$ can fall into one of the four distinct sets: $X_{i+3},$ $Z'_i$, $Z''_i$, $G\backslash{X_{i-3}}$. We prove the lemma for each case.

If $s\in X_{i+3}$, let $X'_i=X_i-\{v|v\in Z'_i \ \&\  v \text{ has more than one outgoing residual arc}\}$. $X'_i$ contains all the deficit nodes. 
For each arc that crosses $X'_i$, it either originates from a vertex in $Z'_i$ that has only one outgoing arc, or ends at a vertex in $Z'_i$ with more than one outgoing arc but only one incoming arc. Hence totally we have at most $|Z'_i|$ arcs cross $X'_i$ and the total excess is at most $6n/(\triangle-6)$ by Lemma \ref{lemma:excess}.

If $s\in Z'_i$, let cut $X'_{i-3}=X_{i-3}-\{v|v\in Z''_i \ \&\  v \text{ has more than one outgoing residual arc}\}$. $X'_{i-3}$ contains all the deficit nodes. 
For each arc that crosses $X'_{i-3}$, it either originates from a vertex in $Z''_i$ that has only one outgoing arc or ends at a vertex in $Z''_i$ with more than one outgoing arc but only one incoming arc. Hence totally we have at most $|Z''_i|$ arcs cross $X'_{i-3}$ and the total excess is at most $6n/(\triangle-6)$ by Lemma \ref{lemma:excess}.

For $s\in G\backslash{X_{i-3}}$ or $s\in Z''_i$, the proofs are very similar as the above two conditions. In summary, the total excess is at most $O(n/\triangle)$ when the first phase ends.
\end{proof}

With the bound of $\tau$, we still need to know the bounds of the \texttt{step 1} and \texttt{step 2}. Here we show they are both linearly related to the graph size.

\begin{lemma}\label{lemma:bnd4steps}
On unit-capacity graphs, the bounds of \texttt{step 1} and \texttt{step 2} are both $O(m)$.
\end{lemma}
\begin{proof}
The proof of \texttt{step 1} can be found in section 3 of \cite{unitCapMCC}. As mentioned in the proof of Lemma \ref{lemma:numIt}, the For loop of \texttt{step 2} will be broken after $O(m)$ pushes in our implementation. Thus, the bounds of \texttt{step 1} and \texttt{step 2} are both $O(m)$.
\end{proof}

Now, we can summarize the complexity of one call of \texttt{RESTORE}.
\begin{lemma} \label{lemma:restore}
The bound of one call of \texttt{RESTORE} is $O(n^{1/2}m)$.
\end{lemma}
\begin{proof}
By setting $\triangle=\sqrt{n}$, the number of iterations needed for one call of \texttt{RESTORE} is $O(\sqrt{n})$. Based on Lemma \ref{lemma:bnd4steps}, each iteration takes $O(m)$ time, so the bound for \texttt{RESTORE} is $O(n^{1/2}m)$. 
\end{proof}
Now we give the complexity analysis of price refinement (line 26-29 in Alg. \ref{alg:cs2}). 
Price refinement is trying to get the $\epsilon/2$-optimal circulation by only manipulating the price function. This is possible when there is no cycle in the admissible network $G_A$. Finding a cycle in a directed graph can be fulfilled in $O(m)$ using depth first search. If $G_A$ is a DAG, the bound of one iteration of price refinement is also $O(m)$ \cite{sp_scaling}. 
In our implementation, we will give up the price refinement if the number of iterations exceeds $\sqrt{n}$. Thus, the bound for price refinement is $O(n^{1/2}m)$, which is the same as \texttt{RESTORE} and will not dominate the time consumption of the framework. 
Note that this early stopping strategy will not influence the correctness of our algorithm. Because the price refinement is only applied on price function $p$, no imbalanced node is created throughout the price refinement. Also because none of the operations in price refinement decreases the reduced costs of admissible arcs or decreases the reduced costs of other arcs more than $\epsilon$, no matter when we break the price refinement, the flow is still guaranteed to be an $\epsilon$-optimal circulation. 

Below we give the overall complexity of our framework.
\begin{theorem}
Our minimum-cost circulation-based framework solves the identity inference problem in MOT in $O(n^{1/2}m\text{log}(nC))$.
\end{theorem}
\begin{proof}
As mentioned before, the cost scaling method guarantees to converge to optimal solution by calling $\text{log}(nC)$ times of \texttt{RESTORE}. Thus, the bound for the main loop is $O(n^{1/2}m\text{log}(nC))$. Other than the main loop, the other panels of our framework take linear time, so the overall bound of our framework is $O(n^{1/2}m\text{log}(nC))$.
\end{proof}

Note that our proofs are enlightened by \cite{unitCapMCC}, which proves the bound of a pure blocking flow-based cost-scaling algorithm on solving minimum-cost flow on unit-capacity graphs. However, in pure blocking flow-based algorithm, the set of imbalanced nodes keeps shrinking and no new imbalanced nodes are created throughout the $\texttt{RESTORE}$, while in our algorithm, the push/relabel operations create new excess nodes. Thus, we cannot prove the bound using this specialty of the imbalanced nodes as did in \cite{unitCapMCC}. Luckily, we found that push/relabel will only happen on excess nodes and stop once the nodes become balanced and thus the set of deficit nodes keeps shrinking. We leverage this specialty and prove that the best ever bound can be achieved using our algorithm. 

Compared with cs2 \cite{goldberg1997cs2}, the idea of \texttt{set-relabel} and \texttt{price refinement} in our implementation are both adapted from it. However, cs2 did not explicitly use the blocking flow created by \texttt{set-relabel} to guide flow pushing. Thus, it is not guaranteed to decrease the total excess of the network within $O(m)$ pushes. In our implementation, we force the flow to be firstly pushed along the blocking flow to make sure we can decrease the total excess of the graph at least one. Besides, we set an upper bound for the iteration number of \texttt{price refinement}, which otherwise would dominate the overall theoretical bound and make it $\sqrt{n}$ folds worse.

\begin{table*}[t]
\begin{threeparttable}
\centering
\caption{Efficiency Comparison on KITTI-Car Datasets}
\begin{tabular}{lllllllll}
\hline
Datasets
      & \multicolumn{4}{c}{KITTI(DPM)}                                                                               & \multicolumn{4}{c}{KITTI(reglets)}                                                                            \\\hline
     
& \multicolumn{1}{l}{seq00} & \multicolumn{1}{l}{seq10} & \multicolumn{1}{l}{seq11} & \multicolumn{1}{l}{seq14} & \multicolumn{1}{l}{seq00} & \multicolumn{1}{l}{seq10} & \multicolumn{1}{l}{seq11} & \multicolumn{1}{l}{seq14} \\\hline

\multicolumn{9}{l}{(a) affinity model from \cite{ssp4mot}}                                                                                                                                                            \\
SSP   & 18.4(72)                & 142.8(76)                & 68.0(74)                 & 120.8(173)                & 6.2(57)                  & 9.1(76)                  & 9.4(81)                  & 31.8(172)                 \\
dSSP  & 3.4(13)                  & 12.7(7)                  & 6.4(7)                   & 12.0(17)                  & 1.4(13)                  & 1.8(15)                   & 1.5(13)                   & 6.4(34)                   \\
MCF   & 11.8(46)                 & 88.4(47)                 & 42.6(46)                 & 41.1(59)                 & 4.3(39)                  & 7.2(60)                  & 5.7(50)                   & 10.3(55)                  \\
\textbf{CINDA} & \textbf{0.26(1)}         & \textbf{1.90(1)}          & \textbf{0.92(1)}          & \textbf{0.70(1)}          & \textbf{0.11(1)}         & \textbf{0.12(1)}          & \textbf{0.12(1)}          & \textbf{0.19(1)}          \\\hline

\multicolumn{9}{l}{(b) affinity model from \cite{followme}}                                                                                                                                                          \\
SSP   & 20.9(9)                & 266.9(342)                & 77.9(97)                 & 96.1(356)                 & 9.4(72)                 & 12.8(111)                 & 18.1(91)                 & 54.7(521)                 \\
dSSP  & 3.2(1)                  & 34.4(44)                  & 9.6(12)                   & 12.0(44)                  & 3.8(29)                  & 3.1(27)                   & 5.9(30)                   & 15.3(146)                 \\
MCF   & 13.6(6)                 & 98.6(126)                 & 45.1(56)                 & 39.8(147)                 & 5.0(39)                  & 6.0(52)                   & 6.3(31)                   & 8.9(85)                   \\
\textbf{CINDA} & \textbf{2.21(1)}         & \textbf{0.78(1)}          & \textbf{0.80(1)}          & \textbf{0.27(1)}          & \textbf{0.13(1)}         & \textbf{0.12(1)}          & \textbf{0.20(1)}          & \textbf{0.11(1)}          \\\hline

\multicolumn{9}{l}{(c) affinity model from \cite{bydpixels}}                                                                                                                                                            \\
SSP   & 1.2h(8.6k)              & 10.0h(7.6k)              & 4.5h(8.0k)               & 5.2h(11.6k)               & 437.6(1.9k)              & 235.8(1.0k)               & 246.8(968)               & 605.9(1.8k)               \\
dSSP  & 917.9(1.8k)              & 3.7h(2.8k)               & 1.0h(1.7k)                & 0.9h(2.1k)                & 198.4(863)              & 224.5(976)               & 242.2(950)               & 649.1(1.9k)               \\
MCF   & 24.6(49)                & 184.3(39)                & 78.7(39)                 & 69.7(43)                 & 5.9(26)                  & 8.5(37)                   & 7.7(30)                   & 12.8(38)                  \\
\textbf{CINDA} & \textbf{0.52(1)}         & \textbf{4.73(1)}          & \textbf{2.03(1)}          & \textbf{1.61(1)}          & \textbf{0.23(1)}         & \textbf{0.23(1)}          & \textbf{0.26(1)}          & \textbf{0.34(1)}        \\\hline 
\end{tabular}
\begin{tablenotes}
      \footnotesize
      \item \textit{* Each item represents the time consumed in seconds (or hours if 'h' is appended) by a specific method under a given experiment setting. Each column represents a video sequence and each row corresponds one method. Bold font indicates the most efficient method. Numbers in the parentheses indicate how many folds the method is slower than the most efficient one.}
\end{tablenotes}
\label{table:effi_KITTI}
\end{threeparttable}
\end{table*}

\begin{figure}[t]
\centering
\includegraphics[width=1\linewidth]{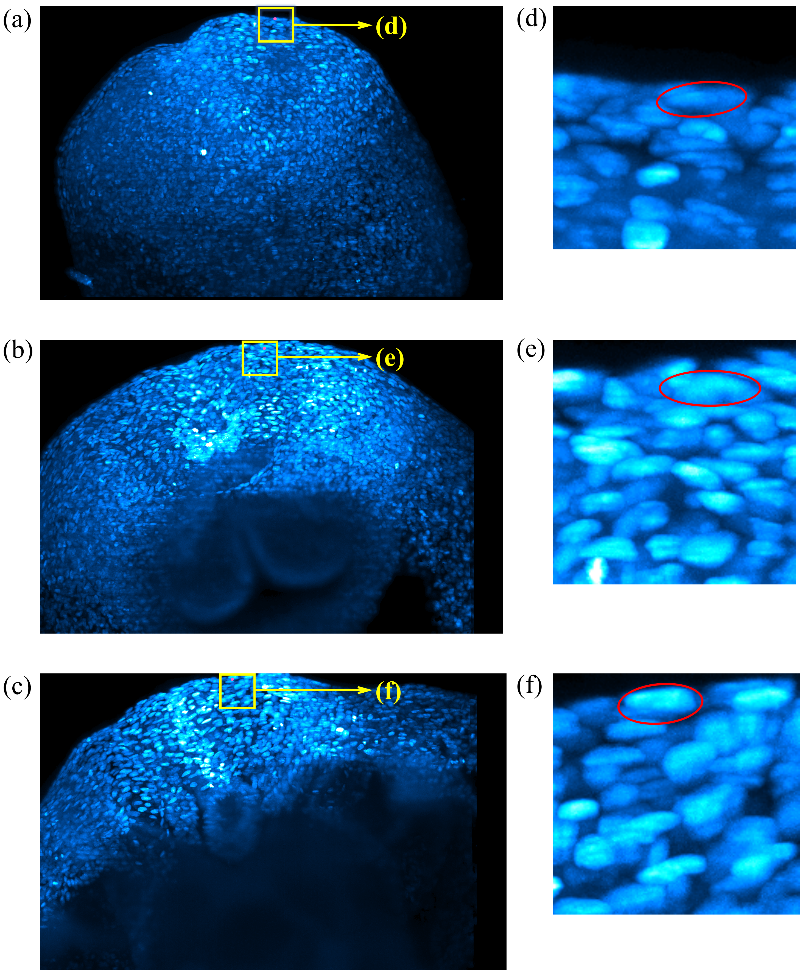}
\caption{
Maximum intensity projections of the mouse embryo data at three different time points ($\sim$20h, $\sim$30h, and $\sim$40h). The zoomed-in segments of the yellow rectangles are on the right side. Red dots in (a-c) indicate one of the tracked cells at these three time points which are also circled on the zoomed-in segments.
}
\label{fig:embryo}
\end{figure}

\section{Experiments}
We use two sets of experiments to study the practical efficiency of our circulation framework. First, we compare the computational time of the proposed circulation-based formulation with previous minimum-cost flow-based formulations on identity inference problems in MOT. Second, we test the efficiency of our formulation as a sub-routine to iteratively approximate the solution of quadratic programming framework in MOT. After efficiency comparisons, we use two applications, car tracking and cell tracking, to illustrate how the efficiency improvement enables more sophisticated but computationally demanding models leading to better tracking performance.

For the efficiency comparisons of identity inference in MOT, we choose three widely used peer algorithms, SSP, dSSP \cite{followme}, and MCF \cite{mcf4mot}. SSP and dSSP are both based on successive shortest paths algorithm, which is known to have superior theoretical complexity over push/relabel-based algorithms on solving the minimum-cost flow problems in MOT when the object number is small \cite{ssp4mot,followme,mcf_kssp}. MCF used cs2 \cite{goldberg1997cs2} for solving their identity inference problems. Written in the C programming language, cs2 is an efficient implementation of the cost-scaling algorithm, which is widely considered as the best solver for generic minimum-cost flow problems. 
SSP were implemented in C++. To make fair comparison, dSSP was re-implemented in C++ from the Python package provided in \cite{followme}. Note that we do not include KSP \cite{mcf_kssp} and the blocking flow-based method 'Block-CS', both of which are listed in Table \ref{table:polyalg4mcf}. KSP is the same as SSP.  For 'Block-CS', it is because the purely blocking flow-based implementation has been widely acknowledged not as practically efficient as cs2 implementation \cite{bkflow_goldberg,goldberg1997cs2}.

After the efficiency comparison, we use two applications in two different scenarios to show the accuracy benefits from our efficient identity inference framework. The first one is car tracking on KITTI-Car dataset \cite{kittiData}, where we incorporate our framework with three affinity models \cite{ssp4mot, followme, bydpixels}. For each affinity model, the identity inference problem is solved either globally by our framework or locally using Hungarian algorithm. The second application is cell tracking on the single-cell level imaging of mouse embryo data \cite{embryo10tb}. Rather than solving the data-association at one step, our framework enables it to iteratively refine the tracking results, which previously is intangible with min-cost flow-based frameworks on large scale data.
\begin{table*}[t]
\centering
\begin{threeparttable}
\caption{Efficiency Comparison on CVPR19 and ETHZ Datasets}
\begin{tabular}{lllllllll}
\hline
\multicolumn{1}{c}{Datasets} & \multicolumn{4}{c}{CVPR19}                                                                    & \multicolumn{2}{c}{ETHZ}                         \\\hline
                           & seq04            & seq06            & seq07            & seq08                                 & seq03            & seq04                          \\\hline

\multicolumn{7}{l}{(a) affinity model from \cite{ssp4mot}}                                                                    \\
SSP                        & 655.6(7)       & 112.4(5)       & 11.3(4)        & 50.6(5)                            & 85.7(85)        & 173.0(179)                     \\
dSSP                       & 0.6h(24)       & 234.9(10)      & 19.8(8)        & 123.9(12)                         & 21.2(21)         & 26.2(27)                      \\
MCF                        & 1.1h(42)       & 799.2(35)        & 104.8(41)        & 358.3(34)                         & 39.4(39)         & 46.1(48)                           \\
\textbf{CINDA}                      & \textbf{93.6(1)} & \textbf{22.9(1)} & \textbf{2.6(1)} & \textbf{10.7(1)}            & \textbf{1.01(1)} & \textbf{0.97(1)}         \\\hline

\multicolumn{7}{l}{(b) affinity model from \cite{followme}}                                                                                      \\
SSP                         & 0.3h(50)        & 168.8(22)       & 18.04(15)                  & 87.7(21)                & 183.9(119)       & 137.3(112)                                \\
dSSP                        & 1.2h(192)         & 520.8(68)         & 32.6(27)                  & 176.2(42)                 & 53.1(34)        & 36.8(30)                              \\
MCF                         & 865.9(40)         & 243.1(32)        & 35.2(29)                    & 135.9(32)                      & 63.1(41)        & 57.8(47)                                \\
\textbf{CINDA}                  & \textbf{21.6(1)} & \textbf{7.6(1)} & \textbf{1.2(1)}          & \textbf{4.2(1)}            & \textbf{1.54(1)} & \textbf{1.23(1)} &              \\ \hline
\end{tabular}
\begin{tablenotes}
      \footnotesize
      \item \textit{* The interpretation of this table is the same as Table \ref{table:effi_KITTI}. The performance of CINDA is consistent with that on KITTI-Car dataset and always the most efficient on CVPR19 and ETHZ datasets.}
\end{tablenotes}
\label{table:effi_others}
\end{threeparttable}
\end{table*}

\subsection{Solving first-order identity inference problem in MOT \label{exp:mcf}}
In this section, we compare the efficiency of our minimum-cost circulation formulation with widely used minimum-cost flow formulations for solving the MAP problem in MOT, where only up to first-order (unary and pairwise) information of detections is considered. We select five public datasets that represent a wide range of real world applications. Three of them are natural image MOT datasets: ETHZ (sequences BAHNHOF and JELMOLI) \cite{ethzData}, KITTI-Car \cite{kittiData}, and MOT CVPR 2019 Challenge (CVPR19) \cite{cvpr19}. The remaining two are the particle tracking dataset, ISBI 2012 Particle Tracking Challenge (PTC) \cite{ptcData} and the 10TB cell tracking dataset, single-cell level mouse embryo data (Embryo) \cite{embryo10tb}. 
For KITTI-Car dataset, we choose four long sequences (seq00, seq10, seq11, seq14) for efficiency comparison; if sequences are too short, usually there is no need to speed up them. For the CVPR19 dataset, we use its test set which includes four videos. Each video contains a crowd of pedestrians. The particle tracking dataset (PTC) provides two different types of simulated data: RECEPTOR and VESICLE, each of which contains 15 videos with different signal to noise ratios and particle densities. We tested on the 'RECEPTOR snr 7' videos with particle density ranging from low to high. The cell tracking dataset (Embryo) contains only one sequence, which is a light-sheet imaging data of mouse embryos that is recorded for 48 hours. The video contains 531 3D frames with total size of around 10 terabytes. The number of cells contained in each frame varies from $\sim$3,000 to $\sim$42,000. Example frames can be found in Fig. \ref{fig:embryo}.

All the experiments are based on the detection results either included in the datasets or provided by the authors of \cite{ssp4mot}, \cite{followme}, \cite{bydpixels}, and \cite{embryo10tb} in their published packages. The KITTI-Car dataset provides detection results from two detectors: DPM \cite{dpm} and reglets \cite{regionlets}. We tested on both of them. The detection results of the ETHZ dataset are also from DPM detector provided by \cite{ssp4mot}. The detection results of CVPR19 dataset were provided by their website. As particle tracking benchmark did not provide detection results, we directly use the ground truth as detections when building the networks. The cell detections on the Embryo data is provided by \cite{embryo10tb} using the method of \cite{tgmm}. The number of detections on the Embryo data is around seven millions. For other datasets, the number of the detections varies from $\sim$7,000 to $\sim$200,000.

For MOT problem with natural image, we use three different affinity models, \cite{ssp4mot},\cite{followme}, and \cite{bydpixels}, to design the arc costs (probabilities in Eq.(10-11)). As \cite{bydpixels} was specifically designed for road scene, it is only applied to KITTI-Car dataset. For particle tracking and cell tracking datasets, we use a probability principled way to design the arc costs. 
For PTC dataset, the probabilities $p_{enter}$ and $p_{exit}$ are learned from the training data, which is defined as the number of trajectories divided by the number of detections. They are the same for all detections, because every particle can suddenly appear or disappear in the filed of view unlike objects in natural images. Since we are using the groundtruth as detection results, we set the linkage cost between the pre-node $o_i$ and post-node $h_i$ of the same object as $c(o_i,h_i)=-(c(s,o_i)+c(h_i,s))$ to make sure the final results will not miss any detection. The transition arc cost between two detections is decided by their distance as we do not have much appearance features to use. We use the training data ('RECEPTOR snr 4') to build an empirical distribution of the distance between two adjacent footprints of a particle. Thus, for the testing data, each distance can be transferred to a p-value $p$ based on this empirical distribution. We set the transition arc cost as $-\log(p)$ directly. 
For Embryo data, we use the similar principles. The difference is that on this dataset we do not have training data, so we can only estimate the arc costs from testing data. For $p_{enter}$ and $p_{exit}$, we estimate the number of new born cells (start of a trajectory) and dead cells (end of a trajectory) by calculating the cell number difference between adjacent frames. For example, if the number of detections decreases $n$ in the next frame, we view it as that $n$ trajectories end. If it increases $n$, we view it as $n$ new trajectories start. This can be used to calculate $p_{enter}$ and $p_{exit}$ using the same way on PTC dataset. For transition arc cost, the key is to build the empirical distance distribution. We estimate it from the distance between each detection and its nearest neighbor. We observe that cells move slowly. For one cell, its footprints in two adjacent frames are highly likely to be nearest neighbors to each other. Thus the empirical distribution of the groundtruth distances can be estimated by calculating the distances between detections and their nearest neighbors in adjacent frames. The linkage arc costs between pre-nodes and post-nodes are set using the same way on PTC dataset. Gating is used in both PTC and Embryo datasets. Each detection is linked to its 3 nearest neighbors in next frame as did in \cite{ptcTracking} on PTC dataset. Because we allow 'jumps' in Embryo dataset, each detection has 3 linkages in each of the following two frames. 

We allow the objects to 'jump', in order to tackle the problem of occlusion or missed detection (other than PTC dataset because we are using groundtruth as detection results). For CVPR19 dataset, density of pedestrian is very high and object occlusion or missing frequently happens. To handle these conditions, we allow objects in frame $t$ to be linked to objects in frames ranging from $t+1$ to $t+15$. The total duration of 15 frames is around $0.5s$. For other datasets, where objects are sparse and occlusion and missed detection happen infrequently (especially for the detection results from DPM or reglets detector which tends to contain many false positives), we link the objects in frame $t$ to objects in two following frames, $t+1$ and $t+2$, in the resultant networks. 

We performed experiments with 40 different settings of detection results, affinity models, and datasets. Summary of the detected objects, frames and the vertices and arcs of the resultant networks can be found in the appendix. Table \ref{table:effi_KITTI}, \ref{table:effi_others}, and \ref{table:PTC_embryo} show the computation time for all experiments. The results were obtained on a workstation with 2.40GHz Xeon(R) E5-2630 CPU and only a single core is used. The other specs of our comparison environment can be found in the appendix. 

\begin{table}[h]
\centering
\begin{threeparttable}
\small
\caption{Efficiency Comparison on PTC and Embryo Data} 
\begin{tabular}{lcccc}
\hline
Method    & SSP               & dSSP              & MCF       & \textbf{CINDA}    \\\hline
PTC-High  & 0.4h(1.4k)        & 1.2h(4.5k)        & 71.5(71)  &\textbf{1.00(1)} \\
PTC-Mid   & 379.2(1.4k)       & 0.3h(4.3k)       & 20.2(72)   & \textbf{0.28(1)} \\
PTC-Low   & 10.7(214)         & 40.4(808)        & 3.0(60)    &\textbf{0.05(1)} \\
Embryo & $>$1d         & $>$1d         & 5.8h(52) &\textbf{398.45(1)}  \\\hline
\end{tabular}
\begin{tablenotes}
      \footnotesize
      \item \textit{* Each row corresponds to one video sequence. SSP and dSSP both fail to solve the identity inference on Embryo data in one day.}
\end{tablenotes}
\label{table:PTC_embryo}
\end{threeparttable}

\end{table}

Each item in the tables represents the time consumed by a specific method under a given experiment setting. More specifically, Table \ref{table:effi_KITTI} shows the results on KITTI dataset. The rows in it are divided three major sections (section (a-c)). Each section represents the efficiency of identity inference incorporated with one affinity model. The columns are divided into two major sections (KITTI(DPM) and KITTI(reglets)), each of which contains the efficiency comparison based on a specific object detector (DPM or reglets). Each row corresponds the performance of a specific method and each column represents the time consuming on a specific video sequence. 
Table \ref{table:effi_others} shows the performance comparison on CVPR19 and ETHZ datasets, whose interpretation is the same as Table \ref{table:effi_KITTI}.  Table \ref{table:PTC_embryo} shows the results on the particle tracking dataset (PTC) and cell tracking dataset (Embryo). Each column here represents a method, while each row corresponds to a video sequence. The particle densities of the three video sequences on PTC dataset are marked out in the names of the video sequences. 
In Table \ref{table:effi_KITTI}, \ref{table:effi_others}, and \ref{table:PTC_embryo}, the time is all reported in seconds, unless when a letter 'h' meaning hour or 'd' meaning day is appended. The numbers in the brackets indicate how many folds the method is slower than the most efficient one. Bold font indicates the most efficient method. 

By formulating the MAP inference problem as a minimum-cost circulation problem, we get the best efficiency in all 40 experiments. For all the experiments, our method takes at most several minutes to finish (mostly seconds), while peer methods need up to several hours. Averaging on all 40 experiments, our circulation-based framework is 53 times faster than MCF. Our method is also much faster than methods based on successive shortest path: averagely we are 1,192 times faster than SSP and 592 times faster than dSSP ('$>$1d' in table \ref{table:PTC_embryo} is viewed as one day).

As an improved version of SSP, dSSP usually performs better than SSP. Interestingly, when the object number is large (\textit{e.g.}, CVPR19 dataset), its performance drops quickly and is even worse than SSP. 
This is because dSSP keeps records of all the nodes of the sequentially instantiated paths and inserts them to the heap used in Dijkstra's algorithm before instantiating new paths. This operation will make the heap size huge when we have a large number of paths. It can be seen from the table that the running times of SSP and dSSP vary greatly with different graph settings. It may take seconds in one graph setting but take several hours for another. 
Based on the same minimum-cost flow framework, MCF performs relatively stable compared with the successive shortest path-based methods SSP and dSSP. It is likely that the local updating strategy of cost-scaling algorithm cs2 scales better with graph size. However, though using a solver that famous for its efficiency for solving generic minimum-cost flow problem, MCF does not always outperform SSP and dSSP. This is true especially when the number of expected targets is small, because the number of expected targets directly decides the number of iterations of SSP and dSSP. For example, using affinity model \cite{followme} on KITTI dataset, we usually will stop with tens of iterations in SSP or dSSP because we are forced to track a small number of targets. Under such condition, we can see the difference between MCF and SSP or dSSP is small. Sometimes MCF is even slower. If we use affinity model \cite{bydpixels}, which relies more on detections and believes each detection should participate in a trajectory, the number of targets to be identified becomes larger. Consequently, MCF becomes much more efficient. As a cost scaling-based method, we are as stable as MCF under different settings as shown in Table \ref{table:effi_KITTI}, \ref{table:effi_others}, and \ref{table:PTC_embryo}. Because our framework does not need to search the optimal trajectory number, orders of efficiency improvement can always be achieved.
 
\subsection{Solving high-order identity inference model in MOT}
High-order relationships between detected objects have been incorporated to further improve the accuracy of tracking in MOT problems \cite{chari2015pairwise,fusionHeadQuadratic,lagrangian1, rank1tensor}. This kind of task is commonly formulated as a quadratic programming problem.
However, identity inference based on such formulation is NP-hard. Existing methods approximate the solution with the help of Frank-Wolfe algorithm or Lagrangian relaxation, where minimum-cost flow identity inference solvers are frequently used as a sub-routine and called iteratively. However, since the number of targets is also unknown, the minimum-cost flow-based framework will encounter the same problem as we mentioned before. Under such circumstances, our circulation-based framework is a better alternative.
\begin{table}[h]
\small
\caption{Efficiency Comparison of Solving Quadratic Programming Problem}
\centering
\begin{tabular}{lcccc}
\hline
Method    & SSP               & dSSP              & MCF       & \textbf{CINDA}    \\\hline
PETS S1.L1-2 & 391.9(14)         & 557.0(20)         & 643.1(22) &\textbf{27.5(1)}  \\\hline
\end{tabular}
\label{table:qua}
\end{table}

We tested the efficiency of our framework in solving the identity inference problem formulated by quadratic programming \cite{chari2015pairwise} on the dataset PETS09 S1.L1 \cite{pets09Data}. 
The quadratic objective function is provided by the authors in their software package derived from the sequence 'Time\_13-59/View\_002'. 
After relaxing the integer solution constraints, the Frank-Wolfe algorithm iteratively solves this quadratic programming problem. In each iteration, the problem is reduced to an ILP problem with the same form of Eq.(7-9). Thus, this sub-routine problem can either be solved using minimum-cost flow-based frameworks or our circulation-based framework. The flow network has 5866 vertices and 36688 arcs. The circulation network has 5865 vertices and the same number of arcs. The step size of the Frank-wofle algorithm is $k/(k+2)$, where $k$ is the index of the current iteration. The results are shown in Table \ref{table:qua}. Our circulation-based framework is still much more efficient than minimum-cost flow-based ones: we are 22 times faster than MCF, 14 times than SSP and 20 times than dSSP. Note that for this quadratic problem, efficiency improvement from our framework is not as great as solving directly the first-order identity inference problem, especially compared with SSP and dSSP.
This is a result of the small number of targets in the video as we mentioned in section \ref{exp:mcf}, which benefits successive shortest path-based algorithms for converging with less iterations. In this test video, we have averagely less than 10 detections in each frame and totally only 44 targets appear throughout the 241 frames.

\subsection{Case study: tracking using min-cost circulation framework}
\begin{figure*}[t]
\centering
\includegraphics[width=1\linewidth]{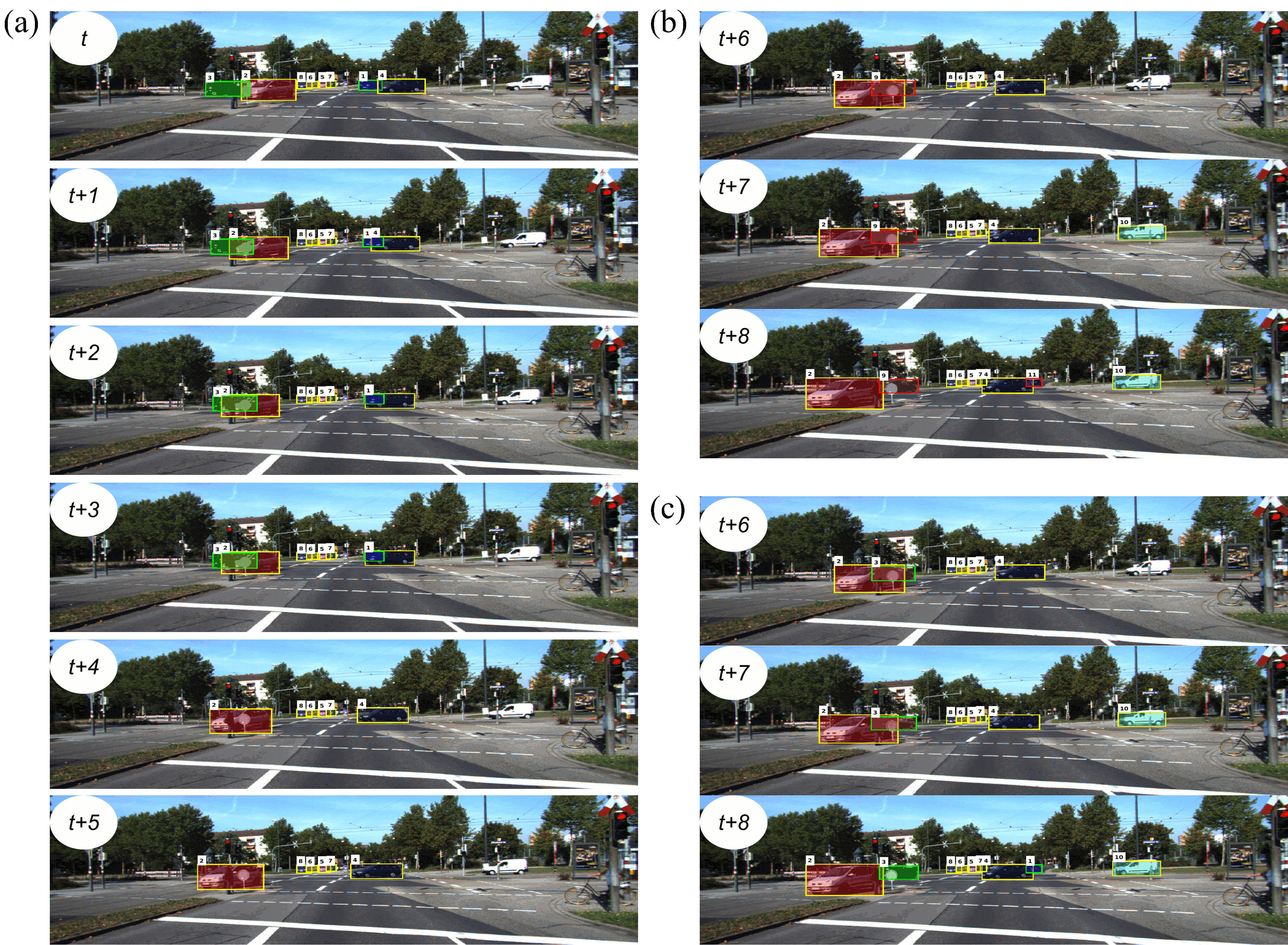}
\caption{
Qualitative results using local optimal (panel (b)) and global optimal (panel (c)) inference models. As shown in panel (a), starting from frame $t+4$, car \#1 and \#3 (with green bounding boxes) were occluded by other cars and thus missed in the following frames. Using local optimal model, ID switches can not be avoided when they re-appear after occlusion. As is shown in panel (b), car \#1 and \#3 are viewed as two new cars (with red bounding boxes). Global inference strategy can retrieve the identity of the missed cars as we use information from previous 5 frames as shown in (c).
}
\label{fig:kitti}
\end{figure*}

We present two case studies using our efficient inference model in two tracking scenarios to demonstrate the consequence of the improved efficiency. One is for car tracking on KITTI-Car dataset. The other is the cell tracking on an Embryo dataset. The first case shows that by accumulating information from the whole video, the tracking accuracy of current state-of-the-arts can be further improved. Our efficient inference model greatly relieves the computational concern of getting global optimal solution. Second, because our framework is efficient, it is computationally feasible to iteratively refine tracking results on large data like Embryo. This is valuable especially on particle or cell tracking problems, where discriminative features between detections are limited. In the second case, improved tracking accuracy is obtained by iteratively refining the tracking results using the features estimated from the tracking results of previous iterations.

\subsubsection{Car tracking on KITTI-Car dataset}
We compare the overall tracking accuracy based on local and global identity inference strategies. Three different network flow-based trackers were used \cite{ssp4mot}, \cite{followme}, and \cite{bydpixels}. 
The experiments were performed on the KITTI-Car training set \cite{kittiData}.
Trackers \cite{ssp4mot} and \cite{followme} both relies on the same MAP formulation described in section 3 for identity inference, but each tracker proposed its own affinity model.
Their detection results both come from the DPM detector \cite{dpm} provided by the KITTI benchmark. Tracker \cite{bydpixels} originally used a greedy framework which linked adjacent frames each time. In \cite{bydpixels}, because they used deep learning-based detector RRC \cite{RRC_detector}, which is more robust on car detection, each detection was believed to participate in an output trajectory. This means the probability $\beta_i$ for detection $x_i$ is approaching to zero. Under such condition, the entering and exiting arc costs become meaningless and can be directly set as zero. The transition arc cost $C_{i,j}$ is set as the negative value of the similarity between $x_i$ and $x_j$.

\begin{table}[]
\centering
\begin{threeparttable}
\caption{Accuracy Comparison on KITTI-Car Dataset}
\begin{tabular}{l|cc||cc||cc}
\hline
              & \multicolumn{2}{c||}{\cite{ssp4mot}} & \multicolumn{2}{c||}{\cite{followme}} & \multicolumn{2}{c}{\cite{bydpixels}} \\ \hline
              & LOC                        & GLB                      & LOC                       & GLB                        & LOC                        & GLB                        \\ \hline
MOTA {[}\%{]} & 38.8                     & \textbf{41.0}            & 46.7                     & \textbf{48.1}             & 81.2                      & \textbf{83.0}             \\
MOTP {[}\%{]} & 78.8            & 78.8                     & 78.6                     & \textbf{79.0}             & 90.1                      & \textbf{90.3}             \\
IDF1 {[}\%{]} & 40.3                     & \textbf{42.6}            & 57.1                     & \textbf{57.3}             & 75.9                      & \textbf{80.4}             \\
FAR           & \textbf{0.06}            & 0.07                     & 0.16                     & \textbf{0.14}             & \textbf{0.23}                      & 0.25             \\
MT {[}\%{]}   & 7.3                       & \textbf{8.1}              & \textbf{15.7}              & 12.4                & 79.8                       & \textbf{81.2}                       \\
ML {[}\%{]}   & 45.9                      & \textbf{44.7}             & 40.1              & \textbf{28.1}                 & 2.2                        & \textbf{2.1}                        \\
IDS           & 284                      & \textbf{276}             & \textbf{138}                      & 153               & 514                       & \textbf{355}              \\\hline
\end{tabular}
\begin{tablenotes}
      \footnotesize
      \item \textit{* 'LOC' means local inference model, which only considers linking between consecutive frames each time. 'GLB' means global inference model, which using the whole video for linking detections.}
\end{tablenotes}
\label{table:accuracy}
\end{threeparttable}
\end{table}

For all these three methods, we tested both local and global optimal inference models. The local model uses greedy linking strategy, which considers two consecutive frames each time. This is widely used in many online tracking scenarios \cite{bydpixels, LSST17}. The linking between two frames is solved by Hungarian algorithm ('LOC' in Table \ref{table:accuracy}). The global model formulates the whole video as a single network flow problem and the solution has global optimality. Note that if the MAP problem statement and the similarity score design are the same, the circulation network and flow network give identical and optimal results. Therefore, we only state the results from our circulation network ('GLB' in Table \ref{table:accuracy}). As in section \ref{exp:mcf}, to tackle the problem of occlusion or missed detection, in global formulation, we allow the objects to 'jump'. For DPM detector-based trackers \cite{ssp4mot,followme}, we allow only one jump because the large amount of redundant detections. For RRC-based tracker \cite{bydpixels}, we link the objects in frame $t$ to objects in frame $t+1$ to $t+5$ in the resultant networks, which is equivalent to 0.5s gap. 
There are also plenty of other local optimal solutions for the MAP inference models (\textit{e.g.} using dynamic programming \cite{ssp4mot}, batch processing \cite{followme}), which has already been shown to degrade the performances compared with global optimal solution in the corresponding papers. Note that the measures we used were provided by \cite{mot16}, which was stricter than those defined originally in \cite{clear2006}.

As shown in Table \ref{table:accuracy}, global optimal solutions are always better than their local optimal counterparts in terms of 'MOTA', the most widely accepted metric \cite{mot16}, and 'IDF1'. More importantly, when the object detectors like RRC miss some detections \cite{bydpixels}, global optimal solution enjoys lower rate of ID switches (IDS) due to the flexible linkage choices. This is a natural advantage of global solvers and is also important for applications like object re-identification. Our global optimal circulation-based solver achieves a throughput of 1525, 1551, and 2930 fps based on these three affinity models respectively. For \cite{ssp4mot} and \cite{followme}, DPM detector has high portion of false positives ($\sim$80\%) and low missing rate. Under such circumstances, ID switches happen infrequently with carefully designed affinity measures. 
It is worth mentioning that \cite{bydpixels} is the top performing tracker on KITTI benchmark, and it still can be improved by simply shifting from the local framework to a global framework.

\subsubsection{Cell tracking on Embryo dataset}
Our efficient identity inference framework makes it tangible to iteratively refine the tracking results even for data as large as the single-cell level imaging of mouse embryo. The computational time of global identity inference on the mouse embryo data is decreased from at least several hours to several minutes. The first benefit of an iterative framework is that the affinity measures become not limited to pairwise features between detections, because we have sets of trajectory hypotheses from previous iterations. This is especially helpful for spot tracking or cell tracking problems, where discriminative pairwise features are not reliable and quite limited. The commonly used pairwise relationships like overlapping ratio or center distance can be biased because of the unknown velocities. The second advantage is that we can better estimate the parameters used in the whole tracking framework. This is important when training data is not available.
Here we propose an iterative framework on the Embryo data to show that these two benefits improved the overall tracking accuracy  dramatically.

\begin{table}[]
\centering
\caption{Accuracy Comparison on Mouse Embryo Data}
\begin{tabular}{c|ccccc}
\hline
             & Iter \#1     & Iter \#2  & Iter \#3  & Iter \#4     & Iter \#5     \\\hline
MOTA{[}\%{]} & 3.3          & 10.0         & 12.0      & 13.5        & \textbf{13.6} \\
IDF1{[}\%{]} & 57.1         & 61.2         & 61.4      & 62.1        & \textbf{62.1} \\
FAR          & 40.7         & 39.0         & 37.6      & 36.8        & \textbf{36.7} \\
MT{[}\%{]}   & 70.0         & 70.0         & 70.0      & 70.0        & 70.0          \\
ML{[}\%{]}   & 9.0          & \textbf{5.0} & 9.0       & 9.0         & 9.0           \\
IDS          & 72           & 18           & 16        & \textbf{15} & \textbf{15}  \\\hline
\end{tabular}
\label{tbl:embryoAccu}
\end{table}

We iteratively refine the affinity designs on Embryo data described in section \ref{exp:mcf} in two directions. The first direction is to recruit more trajectory-related features for affinity measure. 
One feature is the velocity estimated from the set of trajectory hypotheses. At the very beginning, we use distances among detections because we did not known the velocity \textit{in priori}. Now in the iterative framework, for each detection, we can estimate its expected location in the next frame based on its instant velocity. Then distance can be calculated using this predicted location and other detections in the following frames. Besides, for each cell, its instant velocity is calibrated using its four nearest neighbors' velocity (by calculating the median of these five cells). The other feature is the probability of a jump happening $p_{jump}$ in a trajectory hypotheses. In non-iterative framework, this is not accessible. Now we can estimate it as the number of jumps divided by the number of all linkages and thus add a punishment to jumping arcs, whose cost becomes $-\log(p*p_{jump})$. The second direction is to refine the parameters used in the affinity measures. Since we have the set of trajectory hypotheses, we can refine the probability $p_{enter}$ and $p_{exit}$ as we did in PTC dataset, as well as the distribution of the empirical distances.

To quantitatively measure the performance, we manually labelled 100 cell trajectories in 40 time points of the mouse embryo data which contains totally 2040 cell footprints. The results of iterative framework are shown in table \ref{tbl:embryoAccu}. The results of 'Iter \#1' denote the results from the original non-iterative framework. We can see that the iterative framework outperforms non-iterative framework in term of almost all the metrics, especially for the number of ID switches, which decreases from 72 to 15. This is mainly due to that more features can be utilized, which makes the affinity measure more discriminative. Since there are other true positive trajectories that have not been labelled in the field of view, we did not consider the trajectory hypotheses that have no overlap with any groundtruth when calculating the accuracy. The more iterations, the better the performance can be in this case. After five iterations, the results almost converge with little accuracy increase. Note that our algorithm takes only half an hour to finish these five iterations while the other solvers need more than one day.

\section{Discussion and conclusion}
In this paper, we have presented a minimum-cost circulation-based framework for efficient data-association/identity inference in MOT. So far, the minimum-cost flow framework has been most successful for this kind of problem, but it has a high computational burden, which not only hinders the trackers from using a sufficient number of frames to make accurate linking but also holds up the potential of iterative tracking refinements. Our framework leads to the best theoretical complexity bound of solving this identity inference problem and achieves orders of empirical efficiency improvement, which is demonstrated on a wide range of public datasets.

The implementation of our framework combines ideas from the blocking flow-based cost-scaling algorithm \cite{unitCapMCC}, which has the best known theoretical bound $O(\text{min}\{n^{2/3}, m^{1/2}\}m\text{log}(nC))$ for solving minimum-cost circulation problem on unit-capacity graph, and from the most practically efficient implementation of cost-scaling algorithm, cs2 \cite{goldberg1997cs2}, which, however, has an inferior theoretically complexity of $O(nm\log(nC))$ for the same problem.
Our implementation takes advantage of the theoretical efficiency from the blocking flow algorithm and the practical efficiency from cs2. 
While \cite{unitCapMCC} mentioned that the bound of $O(n^{1/2}m\text{log}(nC))$ may also be achievable by blocking flow-based implementation, no rigorous proof was provided. More importantly, the purely blocking flow-based implementation has been widely acknowledged not as practically efficient as cs2 \cite{bkflow_goldberg,goldberg1997cs2}. Indeed, our own experiments show that it is averagely $\sim$3 times slower. On the other hand, compared with cs2, our implementation has similar practical efficiency but has much better theoretical bound with an improvement of $\sqrt{n}$ folds. Note that the bound $O(nm\log(nC))$ was based on the property of the unit-capacity graph and the much looser bound given in \cite{mcf4mot,ssp4mot} was based on the generic graph.

Our efficient framework should benefit both offline and online multi-object tracking applications. For offline scenarios, it makes it possible to build and infer global models based on large scale data sets. For online applications, it allows linking with a longer history of frames and obtains better results while still achieving the real-time speed. In addition, our work will also encourage the development and application of more accurate tracking models that are otherwise limited to smaller data due to the higher computational cost. One example is a model that accumulates evidence from more frames and considers higher order interactions. Another example is a model that iteratively refines the affinity scores based on the tracking results obtained in previous iterations.

In the future, we expect the efficiency of solving the identity inference problem can be further improved by incorporating a good initialization for the minimum-cost circulation problem. This initialization can be obtained by either using greedy linking strategies that only optimize linkages among detections in two consecutive frames or by taking a more extreme approach that simply links each detection to its nearest neighbor in the next frame. If the initialization is good enough to cover most of the correct linking results, we can simply fine-tune the network until reaching the optimality. This procedure should be more efficient than solving the circulation problem from scratch. We believe that efficiency improvement will free up new opportunities. It would be interesting to see how this more efficient framework can push forward the whole MOT field and beyond. We have shown in this report that a better affinity model can be built when it is affordable to perform multiple iterations of tracking. As object detection is an indispensable component for MOT problem and an important task for its own sake, we expect that the use of iterative tracking results could further improve both detection and tracking results.


%

\appendices
\section{Details about experiments}
\subsection{Implementation details}
SSP and dSSP are both implemented in C++. The key step for these two methods is the implementation of the Dijkstra's algorithm for shortest path searching. We use the data structure of self-balanced binary search tree to implement the Dijkstra's algorithm which uses $O(1)$ time for popping top element and $O(\log(n))$ for pushing a new value.



All comparisons were conducted on Ubuntu 16.04 LTS with the code compiled by g++ v5.4.0. The CPU is a 2.40GHz Xeon(R) CPU E5-2630, but only a single core is used. The RAM size is 128GB and the memory speed is 2133MHz.


\begin{table*}[]
\centering
\caption{Details of KITTI-Car dataset}
\begin{tabular}{lllllllll}
\hline
Datasets     & \multicolumn{4}{c}{KITTI(DPM)}    & \multicolumn{4}{c}{KITTI(reglets)} \\\hline
             & seq00  & seq10  & seq11  & seq14  & seq00  & seq10  & seq11  & seq14   \\
\#frames     & 465    & 1176   & 774    & 850    & 465    & 1176   & 774    & 850     \\
\#detections & 51100  & 181132 & 104748 & 96974  & 19885  & 22189  & 24524  & 35198   \\\hline
\multicolumn{9}{l}{(a) graph design from \cite{ssp4mot}}                             \\
\#vertices   & 102201 & 362265 & 209497 & 193949 & 39771  & 44379  & 49049  & 70397   \\
\#arcs       & 171135 & 608881 & 349869 & 325256 & 71730  & 78273  & 86864  & 123008  \\\hline
\multicolumn{9}{l}{(b) graph design from \cite{followme}}                            \\
\#vertices   & 102201 & 362265 & 209497 & 193949 & 39771  & 44379  & 49049  & 70397   \\
\#arcs       & 173242 & 609576 & 352140 & 328352 & 71977  & 78524  & 87048  & 122875  \\\hline
\multicolumn{9}{l}{(c) graph design from \cite{bydpixels}}                             \\
\#vertices   & 102201 & 362265 & 209497 & 193949 & 39771  & 44379  & 49049  & 70397   \\
\#arcs       & 172317 & 607035 & 350276 & 326635 & 71592  & 78477  & 86830  & 122763 \\\hline
\end{tabular}
\label{tb:graphDesign1}
\end{table*}

\begin{table}[]
\centering
\scriptsize
\caption{Details of CVPR19 and EHTZ datasets}
\begin{tabular}{llllllllll}
\hline
Datasets     & \multicolumn{4}{c}{CVPR19}        & \multicolumn{2}{c}{ETHZ} \\\hline
             & seq04   & seq06  & seq07 & seq08  & seq03       & seq04          \\
\#frames     & 2080    & 1008   & 585   & 806    & 1000        & 936         \\
\#detections & 208000  & 70189  & 20220 & 43444  & 101180      & 94054             \\\hline
\multicolumn{7}{l}{(a) graph design from \cite{ssp4mot}}                    \\
\#vertices   & 416001  & 140379 & 40441 & 86889  & 202361      & 188109      \\
\#arcs       & 9522469 & 3343504 & 759961 & 1866996 & 358546      & 365461         \\\hline
\multicolumn{7}{l}{(b) graph design from \cite{followme}}                  \\
\#vertices    & 416001  & 140379 & 40441 & 86889  & 202361      & 188109         \\
\#arcs       & 9712807 & 3703396 & 1001361 & 2076685 & 411170      & 368080            \\\hline
\end{tabular}
\label{tb:graphDesign2}
\end{table}

\begin{table}[]
\centering
\caption{Details of PTC and Embryo datasets}
\begin{tabular}{lllll}
\hline
         & \#frames & \#detections & \#vertices & \#arcs   \\ \hline
PTC-High & 101      & 77352        & 154705     & 462213   \\
PTC-Mid  & 101      & 39215        & 78431      & 234438   \\
PTC-Low  & 101      & 7438         & 14877      & 44448    \\
Embryo   & 531      & 6750628      & 13501257   & 60378108 \\ \hline
\end{tabular}
\label{tb:graphDesign3}
\end{table}

\subsection{Specs of the videos and graphs}
The details of the datasets we used in the experiments can be found in Table \ref{tb:graphDesign1}, \ref{tb:graphDesign2}, and \ref{tb:graphDesign3}. For each dataset we show in the table the number of frames and the number of detections in each video. The out-coming graph sizes with respect to different affinity models are also listed. Notice that the the arc numbers can be different with different affinity models.
The numbers listed in the tables correspond to the graphs used in minimum-cost circulation-based frameworks. For any graph used in minimum-cost flow-based framework, the number of vertex should be subtracted by one compared with its corresponding graph used in the circulation-based framework (for more details, see section 4).



\ifCLASSOPTIONcaptionsoff
  \newpage
\fi


\bibliographystyle{IEEEtran}
\bibliography{mcc}

\end{document}